\documentclass[final]{colt2020} 

\title[Covariance semi-bandits]{Covariance-adapting algorithm for semi-bandits with application to sparse outcomes}
\usepackage{times}

\coltauthor{%
 \Name{Pierre Perrault} \Email{pierre.perrault@outlook.com }\\
 \addr Inria Lille \& ENS Paris-Saclay
 \AND
 \Name{Vianney Perchet} \Email{vianney.perchet@normalesup.org}\\
 \addr  ENSAE \& Criteo AI Lab
  \AND
 \Name{Michal Valko} \Email{valkom@deepmind.com}\\
 \addr   DeepMind Paris
}
\usepackage{wrapfig}

\newcommand{\kl}[2]{\mathrm{kl}\pa{#1, #2}}

\newcommand{\N}{{\mathbb N}}
\usepackage{xcolor} 
\definecolor{blued}{RGB}{70,197,221}
\definecolor{pearOne}{HTML}{2C3E50}
\definecolor{pearTwo}{HTML}{A9CF54}
\definecolor{pearTwoT}{HTML}{C2895B}
\definecolor{pearThree}{HTML}{E74C3C}
\colorlet{titleTh}{pearOne}
\colorlet{bull}{pearTwo}
\definecolor{pearcomp}{HTML}{B97E29}
\definecolor{pearFour}{HTML}{588F27}
\definecolor{pearFith}{HTML}{ECF0F1}
\definecolor{pearDark}{HTML}{2980B9}
\definecolor{pearDarker}{HTML}{1D2DEC}
\usepackage{hyperref}   
\hypersetup{
	colorlinks,
	citecolor=pearDark,
	linkcolor=pearThree ,
	urlcolor=pearDarker}
\usepackage{algorithm}
\usepackage{algorithmic}

\usepackage{tikz}
\definecolor{reded}{RGB}{250,150,150}

\definecolor{graphicbackground}{rgb}{0.96,0.96,0.8}
\definecolor{rouge1}{RGB}{226,0,38}  
\definecolor{orange1}{RGB}{243,154,38}  
\definecolor{jaune}{RGB}{254,205,27}  
\definecolor{blanc}{RGB}{255,255,255} 
\definecolor{rouge2}{RGB}{230,68,57}  
\definecolor{orange2}{RGB}{236,117,40}  
\definecolor{taupe}{RGB}{134,113,127} 
\definecolor{gris}{RGB}{91,94,111} 
\definecolor{bleu1}{RGB}{38,109,131} 
\definecolor{bleu2}{RGB}{28,50,114} 
\definecolor{vert1}{RGB}{133,146,66} 
\definecolor{vert3}{RGB}{20,200,66} 
\definecolor{vert2}{RGB}{157,193,7} 
\definecolor{darkyellow}{RGB}{233,165,0}  
\definecolor{lightgray}{rgb}{0.9,0.9,0.9}
\definecolor{darkgray}{rgb}{0.6,0.6,0.6}
\definecolor{babyblue}{rgb}{0.54, 0.81, 0.94}
\definecolor{citrine}{rgb}{0.89, 0.82, 0.04}
\definecolor{misogreen}{rgb}{0.25,0.6,0.0}

\DeclareMathOperator*{\argmax}{arg\,max}

\let\originalleft\left
\let\originalright\right
\renewcommand{\left}{\mathopen{}\mathclose\bgroup\originalleft}
\renewcommand{\right}{\aftergroup\egroup\originalright}

\newcommand{\sset}[1]{\left\{#1\right\}}
\newcommand{\ceil}[1]{\left\lceil#1\right\rceil}
\newcommand{\floor}[1]{\left\lfloor#1\right\rfloor}

\newcommand{\II}[1]{\mathbb{I}{\left\{#1\right\}}}

\newcommand{\Bernoulli}{\mathrm{Bernoulli}}

\newtheorem{assumption}{Assumption}

\newcommand{\R}{\mathbb{R}}

\newcommand{\EE}[1]{\mathbb{E}\left[#1\right]}

\newcommand{\PP}[1]{\mathbb{P}\left[#1\right]}
\newcommand{\Prb}{\mathbb{P}}

\newcommand{\pa}[1]{\left(#1\right)}

\newcommand{\norm}[1]{\left\|#1\right\|}

\newcommand{\abs}[1]{\left|#1\right|}
\newcommand{\imp}{\Rightarrow}

\newcommand{\CommaBin}{\mathbin{\raisebox{0.5ex}{,}}}

\newcommand{\transpose}{^\mathsf{\scriptscriptstyle T}}

\newcommand{\cA}{\mathcal{A}}

\newcommand{\cC}{\mathcal{C}}

\newcommand{\cN}{\mathcal{N}}

\newcommand{\cU}{\mathcal{U}}

\newcommand{\ba}{{\bf a}}

\newcommand{\bb}{{\bf b}}

\newcommand{\bc}{{\bf c}}
\newcommand{\bC}{{\bf C}}
\newcommand{\bD}{{\bf D}}

\newcommand{\be}{{\bf e}}

\newcommand{\bx}{{\bf x}}
\newcommand{\bX}{{\bf X}}

\renewcommand{\epsilon}{\varepsilon}

\renewcommand{\tilde}{\widetilde}
\renewcommand{\bar}{\overline}

\newcommand{\bSigma}{{\boldsymbol \Sigma}}
\newcommand{\bGamma}{{\boldsymbol \Gamma}}
\newcommand{\bmu}{{\boldsymbol \mu}}
\newcommand{\bxi}{{\boldsymbol \xi}}

\newcommand{\nothere}[1]{}

\usepackage{xspace}

\newcommand{\node}[2]{(#1,#2)}

\newtheorem{prop}{Proposition}
\newtheorem{lem}{Lemma}
\newtheorem{cor}{Corollary}
\newtheorem{rem}{Remark}
\newcommand{\arms}{n}
\newcommand{\actions}{\cA}
\newcommand{\vmean}[1]{\bar{\bmu}_{#1}}
\newcommand{\mean}[1]{\bar{\mu}_{#1}}
\newcommand{\meanabs}[1]{\bar{\nu}_{#1}}

\newcommand{\counter}[2]{N_{#1,#2}}
\newcommand{\blambda}{{\boldsymbol \lambda}}

\newcommand{\bzeta}{{\boldsymbol \zeta}}
\newcommand{\bd}{{\boldsymbol d}}

\newcommand{\KL}[2]{\text{KL}\pa{#1\Vert #2}}
\newcommand{\sizecorr}[1]{\makebox[0cm]{\phantom{$\displaystyle #1$}}}

\newcommand{\fA}{\mathfrak{A}}
\makeatletter
\newcommand{\numrel}[2]{
  \refstepcounter{equation}
  \ltx@label{#2}
  \overset{(\theequation)}{#1}
}
\newcommand{\numterm}[1]{\refstepcounter{equation} \ltx@label{#1} (\theequation)}
\makeatother

\begin{document}

\maketitle

\begin{abstract}%
We investigate \emph{stochastic combinatorial semi-bandits}, where the entire joint distribution of outcomes impacts the complexity of the problem instance (unlike in the standard bandits).  Typical distributions considered depend on specific parameter values, whose prior knowledge is required in theory but quite difficult to estimate in practice; an example is the commonly assumed \emph{sub-Gaussian} family. We alleviate this issue by instead considering a new general family  of \emph{sub-exponential} distributions, which contains bounded and Gaussian  ones. We prove a new lower bound on the regret on this family, that is parameterized by the \emph{unknown} covariance matrix, a tighter quantity than the sub-Gaussian matrix. We then construct an algorithm that uses covariance estimates, and provide a tight asymptotic analysis of the regret. Finally, we apply and extend our results to the family of sparse outcomes,
which has applications in
many recommender systems.
\end{abstract}


\section{Introduction}\label{sec:intro}

Complete automatic adaptation of algorithms to the processed data, as opposed to the requirement of prior knowledge on underlying structure or to some manual tuning of parameters, is 
one of the fundamental challenges in machine learning. We address this challenge for \textit{stochastic (combinatorial) semi-bandits}, and provide an algorithm  adaptive to the correlation structure of the data, leading to provably faster learning in a sequential setting with limited feedback.

Stochastic multi-arm bandits (MAB) are decision-making problems where an \emph{agent} sequentially acts in an uncertain environment. At each round $t\in \N^*$, the agent selects an arm $i$ among a ground set $[\arms]\triangleq \sset{1,\dots,\arms}$ of $\arms\in\N^* $ arms. This choice generates  some reward (or outcome) $X_{i,t}\in \R$, a random variable drawn from $\Prb_{X_i}$, independently from previous rounds, where $\Prb_{X_i}$ is some probability distribution --- \emph{unknown} to the agent --- of mean $\mu_i^*$. 
The objective of the agent is to maximize the expected cumulative reward, or equivalently, to minimize the \emph{regret},
defined as the difference between the expected cumulative
reward achieved by always selecting the single
optimal arm and that achieved by the agent. 
To accomplish this objective \citep{robbins1952some}, the agent must trade-off between \emph{exploration} (gaining information about the arm distributions) and \emph{exploitation} (greedily using  the information collected so far).
 To assess the learning policy followed by the agent (also called a \emph{learning algorithm}), upper bounds on the regret are often derived as a guarantee on its performance.
 These bounds are valid provided that $(\Prb_{X_1},\dots,\Prb_{X_n})$ belongs to some family of probability distributions, e.g., the family of sub-Gaussian outcomes.

There exist sophisticated learners  adaptive to the environment, in the sense that  their performance
guarantees improve (or stated otherwise, their regret upper bounds decrease) when the problem instance is ``simpler''  for some appropriate notions of complexity. 
For instance, \citet{audibert2009exploration} and \citet{Mukherjee2017} proposed to estimate the variance of each arm  to construct adaptive confidence intervals for each mean $\mu_i^*$, based on Bernstein's inequality. This leads to an 
algorithm having variance-dependent regret bounds. \citet{garivier2011kl} went beyond variance estimation and proposed a \emph{Kullback--Leibler divergence} based confidence region, and provided a tighter regret upper bound. Thompson
 sampling can also offer such adaptive regret upper bounds \citep{kaufmann2012thompson}. 
Our objective is to attain such adaptivity, but for the challenging combinatorial extension of bandits, called  stochastic semi-bandits, described next.

Henceforth, for notation conveniences, we typeset vectors in bold and indicate components with indices, i.e., $\ba=(a_i)_{i\in [\arms]} \in \R^\arms$. 
In \emph{combinatorial semi-bandits} \citep{cesa-bianchi2012combinatorial}, the action space $\actions$ is a collection of \textit{subset} of arms. At each round $t$, the agent chooses some action $A_t\in \actions$, receives the \emph{total} reward associated to the selected actions $A_t$, assumed to be $\sum_{i\in A_t}X_{i,t}$, and observes the outcome of each base arm of $A_t$, i.e., the vector $\pa{X_{i,t}\II{i\in A_t}}_{i\in [\arms]}$.
The action space $\actions$ depends on the combinatorial problem at hand. For example, actions in $\cA$ could be a path from an origin to a destination in a network \citep{gyorgy07sp,Talebi2013} or a subset of items to recommend to a customer \citep{Wang1997}.
Many other examples and applications are given by \citet{cesa-bianchi2012combinatorial}.  Notice that in this setting, the whole joint distribution of the vector of outcomes is relevant, contrary to standard bandit problems where only the $\arms$ marginals 
are sufficient to characterize the difficulty of the instance. If we define $\bX\triangleq \pa{X_1,\dots,X_\arms}$, the objective is to design a learning algorithm adaptive
 to the distribution $\Prb_{\bX}$. This is more challenging than in standard bandits, where adaptivity is only with respect to~$\otimes_{i\in [\arms]} \Prb_{X_i}$.  

In a first approach, \citet{Degenne2016} considered the general family of $\bC$-sub-Gaussian probability distributions, with $\bC\succeq 0$ (i.e., $\bC$ is positive semi-definite). Formally, those distributions $\Prb_{\bX}$ of mean $\bmu^*$ satisfy
\vspace{-0.32cm}
\begin{align}\forall \blambda\in \R^\arms,~\EE{e^{\blambda\transpose\pa{\bX-\bmu^*}}}\leq e^{\blambda\transpose\bC\blambda/2}.\label{ass:subgau}\end{align}
\vspace{-0.6cm}

\noindent
\citet{Degenne2016} devised an algorithm with a regret bound  depending on the components of another matrix $\bGamma~\succeq~0$,  satisfying $\bGamma\succeq_+ \bC$ (i.e., $\blambda\transpose\pa{\bGamma- \bC}\blambda\geq 0$ for all $\blambda\in \R_+^n$) and $\Gamma_{ij}\geq 0$ for all $i,j$. The major downside is that this algorithm \emph{requires the knowledge} of $\bGamma$. More precisely, their upper bound is of order 
\vspace{-0.25cm}
\begin{align}\frac{\log T}{\Delta}\sum_{i\in [\arms]}\Gamma_{ii}\pa{(1-\gamma)\log^2(m) + \gamma m},\label{rel:degennebound}\end{align}
\vspace{-0.31cm}

\noindent
where $\displaystyle\gamma\triangleq\max_{A\in \actions}\max_{(i,j)\in A^2,i\neq j}\Gamma_{ij}/\sqrt{\Gamma_{ii}\Gamma_{jj}}$ is the maximal off-diagonal correlation coefficient, $\Delta$ is the minimal positive gap between expected total reward of two actions, and $m\triangleq\max\sset{\abs{A},~A\in \actions}\!.$
Interestingly, their regret upper bound highlights regimes interpolating between worst case correlation between outcomes (corresponding to $\gamma=1$) and mutually independent outcomes (where $\gamma=0$).
In particular, the learning rate is much faster in the latter case. 
The main drawback however, is that their approach is not adaptive since the correlation structure of the arms \textit{needs to be given to the agent} (through the matrix $\bGamma$).

Our main objective is to alleviate this issue, and to strive to obtain fast rates for combinatorial semi-bandits, as \citet{Degenne2016} in the case where there is a favorable covariance structure, but \textit{without knowing} it beforehand. Therefore, algorithms should be able to capture the covariance structure given by $\bGamma$ from the data processed and adapt to it. We actually go further by asking whether the matrix $\bGamma$  is the relevant parameter to characterize the difficulty of a problem. We argue that the covariance matrix $\bSigma^*\triangleq \EE{{\pa{\bX-\bmu^*}\pa{\bX-\bmu^*}\transpose}}$ is more pertinent, as it allows to better differentiate complex problems from the easy ones.
One can indeed already argue in favor of a $\bSigma^*$ dependence rather than a $\bGamma$ one, based on the relation $\bSigma^*\preceq_+\bGamma$ (see Appendix~\ref{app:bSigmapreceqbGamma}).

\paragraph{Results and limitations of the results of \citet{Degenne2016}}
 Below, we list the main limitations of the approach of \citet{Degenne2016}: 
\begin{enumerate}
\vspace{-0.15cm}
    \item[$(i)$] The matrix $\bGamma$ needs to be known. This requires specific knowledge about the outcome structure, which is often not precise, as it is usually only  known that outcomes are bounded, or at most that there exists some constant $\kappa$  such that $\kappa^2\geq C_{ii}$ for all $i\in [n]$. The latter is equivalent\footnote{Indeed, $\bC\preceq_+ \bGamma\imp C_{ii}\leq \kappa^2$ for all $i\in [n]\imp C_{ij}\leq \sqrt{C_{ii}C_{jj}}\leq \kappa
  ^2$ for all $i,j\in [n]\imp \sum_{i,j}C_{ij}\abs{\lambda_i}\abs{\lambda_i}\leq \kappa^2\pa{\sum_i \abs{\lambda_i}}^2 \text{ for all }\blambda\in \R^n\imp \bC\preceq_+ \bGamma$. } to $\Gamma_{ij}=\kappa^2$ for all $i,j\in [n]$ and corresponds to the worst case correlation between outcomes ($\gamma=1$) in the regret bound \eqref{rel:degennebound}.
  \vspace{-0.2cm}
    \item[$(ii)$] The value $\gamma$ can be $1$, even when outcomes are only weakly correlated: For instance, if $\arms$ is even, $\bGamma$ can be a block-diagonal matrix with $\arms/2$ blocks of size $2\times 2$ containing only ones.  This scenario can actually occur in many examples; we provide two types below:
    \vspace{-0.2cm}
    \begin{itemize}\item Arms are nodes on a given graph, with some small communities on which outcome tends to be constant \citep{cesa-bianchi2013gang,valko2014spectral,gentile2014online,valko2016bandits}. 
    \vspace{-0.2cm}
    \item Arms are market-basket-like items, with some highly correlated pairs of items (e.g., people buying from category ``books" tend to also buy from category ``CDs",  \citealp{Zhang:2006,He2006}).\end{itemize}
    
    \vspace{-0.4cm}
    \item[$(iii)$] The value $\Gamma_{ii}$ can be high, even for low-variance outcomes, {while intuitively, low variance outcomes should be easy to work with}.
    For example, if  $\bX$ is a binary $1$-sparse random variable  --- as in some recommender systems, where a single item is desired by the user  --- then $X_i\sim \Bernoulli\pa{\mu^*_i}$ with $\sum_{i=1}^\arms \mu^*_i=1$, and $\Gamma_{ii}\geq C_{ii}\geq \pa{\mu^*_i-1/2}/{\pa{\log(\mu^*_i)-\log(1-\mu^*_i)}}$ (and this is tight, see, e.g., \citealp{Buldygin2013}). For $\mu^*_i$ of order $1/\arms$, $\Gamma_{ii}$ is thus at least of order $1/\pa{2\log\arms}$ for $\arms$ large, whereas $\mathbb{V}(X_i)$ is of order $1/\arms$.   
\end{enumerate}
To sum up the arguments above, we claim that (1) knowing a good upper bound on the sub-Gaussian matrix $\bC\preceq_+\bGamma$ is not realistic and (2) even this upper bound is not a good proxy for the complexity  of the instance at hand. 

\paragraph{Contributions}
In this paper, we address the three aforementioned criticisms $(i),$ $(ii),$ and $(iii)$. 
As a consequence, we do not assume that a good upper bound $\bGamma$ on the sub-Gaussian matrix $\bC$ is known, but only that the agent knows that
each marginal $\Prb_{X_i}$ is $\kappa^2$-sub-Gaussian.
 We compensate this relaxation by restricting the distribution family considered through a sub-exponential-type assumption involving the covariance matrix $\bSigma^*$. We argue that this restriction is mild and satisfied 
 by many outcome distributions, including bounded and Gaussian.
 
We characterize the difficulty of the problem with $\bSigma^*$; specifically, we provide a new lower bound, with a dependence on $\bSigma^*$,  more precise than  \citet{Degenne2016}. We also design a new algorithm with matching asymptotic regret upper bound, improving over the state-of-the-art results. One of the key techniques is to build an online adapted estimation of the matrix $\bSigma^*$. 

Our main contribution is in the analysis of this approach, that is not based on the usual \emph{Laplace's method}, which works in the sub-Gaussian framework, but does not handle well our sub-exponential-type assumption. Thus, our analysis is rather based on a \emph{covering-argument} \citep{Magureanu2014}. An important part of our proof is based on the transformation of the axis-unaligned ellipsoidal confidence region associated to a given action $A\in \actions$ into an axis-aligned region, using the following relation $\pa{\Sigma_{ij}^*}_{ij\in A}\preceq_+ \text{diag}\pa{\sum_{j\in A}0\vee\Sigma^*_{ij}}_{i\in A}$. 
 This allows us to conduct the same type of proof than for the independent outcome case (where confidence regions are always axis-aligned), but with a Bernstein-type analysis.\footnote{Remark that contrary to previous work on variance based confidence region, our method can't be easily generalized to Kullback--Leibler divergence based confidence region, since this would require control on higher moments of $\bX$.}
 
We also consider an application of our approach to the family of sparse bounded outcomes: we provide a lower bound on the regret, with an algorithm having a matching asymptotic  regret  upper bound.

\begin{figure}[H]
\centering
\begin{tikzpicture}[scale=2]
\draw[line width=1pt] (-1.5,-1) rectangle (1.5,1);
\draw (0,0) node[ below] {$\vmean{t-1}$};

\draw[scale=1,domain=-90:270,smooth,variable=\x,color=black!50!green,line width=1pt] plot ({1.5*(cos(\x))},{
sin(\x))});
\draw[scale=1,domain=-90:270,smooth,variable=\x,color=black!30!blue,line width=1pt] plot ({1.5*(cos(10)*cos(\x)-sin(10)*sin(\x))},{
sin(50)*cos(\x)+cos(50)*sin(\x))});
\draw[scale=1,domain=-90:270,smooth,variable=\x,color=black!15!red,line width=1pt] plot ({1.5*(cos(-17)*cos(\x)-sin(-17)*sin(\x))},{
sin(50)*cos(\x)+cos(50)*sin(\x))});

\node [black] at (0,0) {\scalebox{.5}{\textbullet}};
\end{tikzpicture}
\caption{\textit{Confidence regions build by  \textsc{escb-c} (the pseudo-ellipse), and \textsc{cucb-kl} (the rectangle), for $\norm{\cdot}_1$ constrained outcomes. Notice that \textsc{cucb-kl} has slightly better confidence intervals along the axis, but that \textsc{escb-c} is better in the direction $\be_{\sset{1,2}}$.}}
\label{fig:square_elli}
\end{figure}
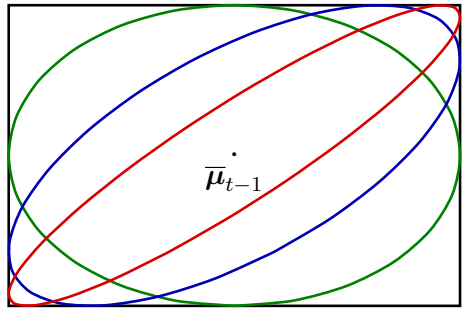

\paragraph{Prior work on stochastic semi-bandits}
We review algorithms for stochastic semi-bandits, coming with the analysis that depends on the family of probability distributions to which $\Prb_{\bX}$ belongs. To begin, \citet{kveton2015tight,chen2015combinatorial}  studied the general family of distributions having sub-Gaussian or bounded marginals. Their algorithms are not adaptive to $\Prb_{\bX}$ and regret bounds depend on parameters characterizing the family, that need to be known 
 (such as the sub-Gaussian constant or a bound on $\norm{\bX}_\infty$). On the other hand, many algorithms are \textit{only} adaptive to marginals of $\Prb_{\bX}$,  either with variance estimates \citep{pmlr-v89-perrault19a,Merlis2019}, or using Kullback--Leibler divergence. 
 These approaches are agnostic to possible correlation between marginals since the confidence region used in their algorithm are always a Cartesian product of confidence intervals (so they are always $\arms$-dimensional hypercubes). As a consequence, this translates into guarantees w.r.t.\ the worst-case correlations quantity possible. Notice that these  algorithms are actually almost direct applications of corresponding classical multi-arm bandits algorithms to the semi-bandit setting. In particular, confidence regions considered are the same in both settings.   

Another line of works restricts the probability distributions family of $\Prb_{\bX}$, so that the dependence existing between arms is controlled.
This conveniently induce better confidence regions valid for distributions in the family, and leads to the development of algorithms based on these regions, having sharper regret upper bounds.
For instance, \citet{combes2015combinatorial} and \citet{sc2018,perrault2020statistical} assumed that $\Prb_{\bX}=\otimes_{i\in [\arms]} \Prb_{X_i}$.
Confidence regions resemble to axis-aligned ellipsoids in this specific case.
They designed UCB (resp. Thompson sampling) based algorithms, leveraging on such tighter ellipsoidal confidence region.
The key difference between the above case is that this time, marginals do characterize the problem, by assumption on the probability distributions family. 

Remark that \citet{Degenne2016} provided a regret bound which adapts to the probability distribution family at hand through the matrix $\bGamma$, although their algorithm is not fully adaptive.
The confidence region used by their algorithm is also ellipsoidal, and depends on the matrix $\bGamma$. This matrix gives the control on the correlations between arms. The confidence ellipsoid is not axis aligned unless $\bGamma$ is diagonal.
To the best of our knowledge, their work is the main competitor in terms of regret bound. 


\paragraph{Sparse bandits}
Independently to combinatorial bandits, there exists a different setting actually dealing with  correlated outcomes in online learning known as \textit{sparse bandits} \citep{KwonPV17,KwonP15,BubeckSparsity2017,yadkori12sparse,carpentier12sparse,gerchinovitz2013sparsity}. The overall idea is to introduce by now a standard sparsity assumption (some parameter vector has only $s$ out of its $\arms$ components that are non zero) into sequential decision making. As usual, the objective is to replace the linear/polynomial dependence in the dimension~$n$ by a linear/polynomial dependence in $s$. Quite interestingly, the sparsity assumption has been studied in two different directions. The fist one assumes that the vector $\bmu^*$ is $s$-sparse, typically in (linear) stochastic bandits \citep{KwonPV17,yadkori12sparse,carpentier12sparse,gerchinovitz2013sparsity}. The second one assumes that the realized vector $\bX_t$ is $s$-sparse, usually in adversarial bandits \citep{KwonP15,BubeckSparsity2017}.

Sparsity in realized outcomes naturally induces negative correlation; this is not necessarily true for sparsity in expectation. More generally both concepts are complementary, since $\bmu^*$ can be sparse with non-sparse realization (for instance, if all $X_{i}$ are i.i.d., equal to  $\pm 1$ with probability $1/2$) and reciprocally (if $\bX$ is a canonical unit vector at random, then its expectation has full support). Surprisingly, 
the sparse outcomes setting has not been investigated in stochastic bandits, even if it lies at the junction of several notions of correlations between outcomes.

\section{Some technical background} 
\label{sec:back}
Let $\be_i$ be the $i^{th}$ canonical unit
vector of $\R^\arms$. The incidence vector of any subset $A\subset [\arms]$ is \(\be_A\triangleq \sum_{i\in A}\be_i.\)
The above definition allows us to represent a subset of $[\arms]$ as an element of $\sset{0,1}^\arms$.
We denote the Minkowski sum of two sets $Z,Z'\subset\R^\arms$ as $Z+ Z'\triangleq\sset{z+z',z\in Z,z'\in Z'}$, and  $Z+ z'\triangleq Z+\sset{z'}$. 
In \emph{stochastic combinatorial semi-bandits}, an agent selects an action $A_t\in \actions$ at each round $t\in\N^*$, and receives a reward $\be_{A_t}\transpose \bX_t$, where $\bX_t\in \R^\arms$ is an unknown random vector of outcomes. The successive vectors $\pa{\bX_t}_{t\geq 1}$ are i.i.d., sampled from $\Prb_{\bX}$, with an unknown mean $\bmu^*\triangleq\EE{\bX}\in \R^\arms$.  
 After selecting an action $A_t$ in round~$t$, the agent observes the outcome of each individual arm in $A_t$. 
Its goal is to
 minimize the regret, defined with $A^*\in\argmax_{A\in \actions}\be_{A}\transpose\bmu^*$ as 
 \vspace{-0.05cm}
\[\forall T\geq 1,\quad R_T\triangleq\EE{\sum_{t=1}^T \pa{\be_{A^*}-\be_{A_t}}\transpose\bX_t}.\]
\vspace{-0.25cm}

\noindent
For any action $A\in\actions$, we define its gap as the difference $\Delta\pa{A}\triangleq \pa{\be_{A^*}-\be_{A}}\transpose\bmu^*$. We then rewrite the regret as $R_T=\EE{\sum_{t=1}^T \Delta\pa{A_t}}$.  We start by stating the assumptions satisfied by $\Prb_\bX$. 
\begin{assumption}[$\kappa^2$-sub-Gaussian marginals] There is a constant $\kappa>0$ (known to the agent) such that $\forall i \in [\arms], \forall \lambda \in \R,~\EE{e^{\lambda\pa{X_i-\mu^*_i}}}\leq e^{\kappa^2\lambda^2/2}.$\label{ass:compo_subgau}\end{assumption}
Assumption~\ref{ass:compo_subgau} is not difficult to satisfy, and does not require any precision on the correlations between outcomes.
In particular, Assumption~\ref{ass:compo_subgau} includes Gaussian outcomes (with variance lower than $\kappa^2$) and bounded outcomes (with $\norm{\bX}_\infty\leq \kappa$). 
We also assume that $\bX$ satisfies the following. 
\begin{assumption}[$\norm{\cdot}_1$-sub-exponential distribution]
 $\forall\blambda\in \R^\arms$ such that $\norm{\blambda}_1\leq 1/\pa{2\kappa}$, we have $\EE{e^{\blambda\transpose\pa{\bX-\bmu^*}}}\leq e^{\blambda\transpose\bSigma^*\blambda}$, where $\bSigma^*\triangleq \EE{\pa{\bX-\bmu^*}\pa{\bX-\bmu^*}\transpose}$ is the covariance matrix of $\bX$.   
 \label{ass:subexp}
\end{assumption}
Importantly, the agent does not know the covariance matrix $\bSigma^*$. 
Remark that Assumption~\ref{ass:subexp} trivially holds for $\bX\sim\cN(\bmu^*,\bSigma^*)$, where $\forall \blambda\in \R^\arms,~\EE{e^{\blambda\transpose\pa{\bX-\bmu^*}}}= e^{\blambda\transpose\bSigma^*\blambda/2}$. The following proposition, proved in Appendix~\ref{app:bounded}, states that it also holds for bounded outcomes. 
\begin{prop}\label{prop:bounded}
If  $\norm{\bX}_\infty\leq \kappa$, then both Assumption~\ref{ass:compo_subgau} and~\ref{ass:subexp} hold. 
\end{prop}
Notice, up to a re-normalization of the regret, we assume w.l.o.g.\  that $\kappa=1$.
\section{Lower bound}
We start by proving in Theorem~\ref{thm:lower} a new gap-dependent lower bound on $R_T$, valid for any covariance matrix $\bSigma^*~\succeq~0$, for some $\Prb_\bX$ satisfying Assumptions~\ref{ass:compo_subgau}~and~\ref{ass:subexp}, some action space $\actions$, and for any consistent algorithm \citep{lai1985asymptotically}, for which the regret on any problem verifies $R_T=o(T^a)$ as $T\to \infty$, for all $a>0$. 
This lower bound demonstrates the link between $\bSigma^*$ and the difficulty of the problem.
It also indicates, in anticipation, that we have to examine a subclass of action sets to hope to improve the upper bound we will provide in Theorem~\ref{thm:upper_cov}.
\begin{theorem}\label{thm:lower}For any $n,m\in \N^*$ such that $n/m\geq 2$ is an integer, any $n\times n$ matrix $\bSigma^*\succeq 0$, any $\Delta>0$, and any consistent policy, there exists an instance with $n$ arms --- characterized by some action space $\actions$, with $m=\max\sset{\abs{A},~A\in \actions}\!,$ some outcome distribution $\Prb_{\bX}$ satisfying Assumptions~\ref{ass:compo_subgau} and~\ref{ass:subexp} with all gaps equal to $\Delta$ and covariance matrix $\bSigma^*$ --- on which the regret satisfies
 \[\liminf_{T\to \infty} \frac{\Delta}{\log\pa{T}}R_T\geq  2\sum_{i\in[n],~i\notin A^*}\max_{A\in \cA,~i\in A}\sum_{j\in A}\Sigma^*_{ij}\cdot\]
\end{theorem}

\noindent
The proof is given in Appendix~\ref{app:lower} and considers $\actions$ containing $\arms/m$ disjoint actions $A_1,\dots,A_{\arms/m}$ composed of $m$ arms, with $A_k=\sset{(k-1)m+1,\dots,km }$, and $\bX\sim \cN(-\Delta/m(\II{i\notin A_1})_i,\bSigma^*)$.
The idea is to make a reduction to some standard bandit problems with $\arms/m$ arms, and to compute the number of rounds $t$ needed to distinguish between $A_k$ and $A_1$. Roughly speaking, $t$ is at least equal to the inverse of the $\text{KL}$ between outcome distributions of $A_k$ and its centered version, and in the case of Gaussian distributions, we get  $t\geq2\mathbb{V}(\sum_{i\in A_k}X_i)/\Delta^2=2\be_{A_k}\transpose \bSigma^*\be_{A_k}/\Delta^2$. It is not surprising that the variance appears, since this can be seen as a measure of the uncertainty we have in our samples: 
The higher the variance, the harder the estimation, and therefore the higher the round~$t$ must be.
Notice that Theorem~\ref{thm:lower} is a refinement of Theorem~1 from \citet{Degenne2016}, in which they consider the same action space $\cA$ but a specific choice for the matrix $\bSigma^*$: it is a block-diagonal matrix with $n/m$ blocks, where each block (corresponding to an action $A$) is equal to $\sigma^2\pa{(1-\gamma)\text{diag}(\be_A)+\gamma\be_A\be_A\transpose}$, i.e., they take the worst case correlation under the controls given by $\sigma^2$ and $\gamma$, and knowing that the problem given by $\cA$ is agnostic to the correlations between the arms of two different blocks.

In the next section, we describe our algorithm \textsc{escb-c} (Algorithm~\ref{algo:cov})
and provide an upper bound on its regret in  Theorem \ref{thm:upper_cov},  where the expression 
 $\max_{A\in \actions,i\in A}\sum_{j\in A}0\vee\Sigma^*_{ij}$  appears. Notice that this is very close to the expression given in Theorem~\ref{thm:lower}. In fact, both expressions coincide when $\bSigma^*$ has only non-negative entries.

\section{Main algorithm and the guarantees}
In this section, we present an algorithm for the setting introduced in Section~\ref{sec:back}. The method is stated as Algorithm~\ref{algo:cov}. To find the action with the highest mean, the agent estimates the mean $\mu_i^*$ of every arm $i$ with their corresponding \emph{empirical averages} defined as
\(\mean{i,t-1}\triangleq\sum_{u\in[t-1]}\frac{\II{i\in A_u}X_{i,u}}{\counter{i}{t-1}},\)
for $t\geq 1$, where $\counter{i}{t-1}\triangleq\sum_{u\in[t-1]}\II{i\in A_u}$ is the number of time arm $i$ have been drawn for the first $t-1$ rounds. As mentioned above, the agent also estimates the covariance $\Sigma^*_{ij}=\EE{X_iX_j}-\mu_i^*\mu_j^*$ of each pair $(i,j)\in [\arms]^2$.  This will be done with the following estimate
\begin{align*}
\bar\Sigma_{ij,t-1}&\triangleq\sum_{u\in[t-1]}\frac{\II{i,j\in A_u}\pa{X_{i,u}-\mean{i,t-1}}\pa{X_{j,u}-\mean{j,t-1}}}{\counter{ij}{t-1}}\\
& = \sum_{u\in[t-1]}\frac{\II{i,j\in A_u}\pa{X_{i,u}X_{j,u}-\mean{i,t-1}X_{j,u}-\mean{j,t-1}X_{i,u}}}{\counter{ij}{t-1}}+\mean{i,t-1}\mean{j,t-1},
\end{align*}
\vspace{-0.3cm}

\noindent
where $\counter{ij}{t-1}\triangleq\sum_{u\in[t-1]}\II{i,j\in A_u}$ is the number of times arms $i$ and $j$ have been drawn \emph{together} for the first $t-1$ rounds.
Notice that in order to efficiently update $\bar\Sigma_{ij,t-1}$, in addition to $\mean{i,t-1}$ and $\mean{i,t-1}$, we only have to maintain the three quantities, 
\[\sum_{u\in[t-1]}\frac{\II{i,j\in A_u}X_{i,u}X_{j,u}}{\counter{ij}{t-1}}\CommaBin\quad\sum_{u\in[t-1]}\frac{\II{i,j\in A_u}X_{i,u}}{\counter{ij}{t-1}}\CommaBin\quad \text{and} \quad\sum_{u\in[t-1]}\frac{\II{i,j\in A_u}X_{j,u}}{\counter{ij}{t-1}}\cdot\]
\vspace{-0.3cm}

\noindent
Using concentration inequalities, we get confidence intervals for the above estimates. We  are then able to use an upper-confidence-bound strategy \citep{auer2002finite}. More precisely, we first build the upper confidence bound on $\Sigma^*_{ij}$ using the fact that $X_i\cdot X_j$ is a \textit{sub-exponential} random variable, since both $X_i$ and $X_j$ are sub-Gaussian by virtue of Assumption~\ref{ass:subgau}.
The result is stated in the following proposition with a proof in Appendix~\ref{app:sigmaupper}.
\begin{prop} With probability $1-10t^{-2}$, we have
$$\abs{\Sigma^*_{ij}-\bar\Sigma_{ij,t-1}} \leq g_{ij}(t)\triangleq16\pa{\frac{3\log(t)}{\counter{ij}{t-1}}\vee\sqrt{\frac{3\log(t)}{\counter{ij}{t-1}}}} + \sqrt{\frac{48\log^2(t)}{{\counter{ij}{t-1}}{\counter{i}{t-1}}}}+\sqrt{\frac{36\log^2(t)}{{\counter{ij}{t-1}}{\counter{j}{t-1}}}}\cdot$$
In particular, defining the upper confidence bound $\Sigma_{ij,t}\triangleq \bar\Sigma_{ij,t-1} + g_{ij}(t),$ it holds that $0\leq \Sigma_{ij,t}-\Sigma^*_{ij}\leq 2g_{ij}(t)$ with probability $1-10t^{-2}$.
\label{prop:sigmaupper}
\end{prop}
To build estimates well concentrated around $\bmu^*$, we will use the matrix $\bSigma_t$ defined above to design the following high probability confidence region for all $A\in\actions$
\begin{align}\cC_t(A)\triangleq \vmean{t-1}\!+\!\left\{\!\sizecorr{\sum_{i\in A}}\bxi\in \R^\arms,
\sum_{i\in A} \frac{\counter{i}{t-1}\xi_i^2}{{\abs{A}\abs{\xi_i}\!+\!\sum_{j\in A}0\vee\Sigma_{ij,t}}}\leq 8( \log t+\log\log t )+4em\right\}\cdot&\label{C_t(A)}\end{align}
The intuition behind this confidence region is similar to the one for empirical Bernstein confidence intervals, but the term $\sum_{j\in A}0\vee\Sigma_{ij,t}$ in the denominator replaces the usual empirical variance. 
To compare our confidence region with the one of \cite{Degenne2016},
notice first that their algorithm uses the matrix $\bGamma$ to build a confidence ellipsoid. They provide an analysis for this confidence ellipsoid using the \emph{Laplace's method} and the matrix relation $\bC\preceq_{+}\bGamma$.
In contrast, our
confidence region is based on the covariance matrix $\bSigma^*$. Our analysis is also different, as we use a \emph{covering-argument} analysis. This is because 
the covariance estimation and Assumption~\ref{ass:subexp} are both hard to handle with {Laplace's method}, that is more appropriate for sub-Gaussian random variables. Indeed, all calculations can be explicit and it is easy to construct a \textit{conjugate prior}. This is \textit{not the case} for  sub-exponential random variables.
Covering arguments are much more easier to use together with a diagonal matrix, so axis-aligned confidence region are desirable. We use an \emph{axis-realignment technique} based on the matrix relation $\pa{\Sigma^*_{ij}}_{ij\in A}\preceq_+ \text{diag}\pa{\sum_{j\in A}0\vee\Sigma^*_{ij}}_{i\in A}$. The upside is to avoid dealing with off-diagonal terms by transforming them into diagonal ones.
From all these previous observations, we can say that the confidence ellipsoid of \cite{Degenne2016} is  tighter as it does not require any axis realignment; however, not only the matrix $\bGamma$ is generally looser  than $\bSigma^*$ but also axis realignment does not alter the analysis, so that our  new approach outperforms theirs in terms of asymptotic regret upper bound.


As common in bandits, the major challenge in the analysis   is to prove that with high probability, $\bmu^*\in\cC_t(A)$ for any action $A\in \actions$.
The covering argument together with the conversion from an {axis-unaligned} confidence region into an {axis-aligned} confidence region allows us to 
 achieve this result (see Lemma~\ref{lem:concentration}).
Therefore, an optimistic estimate $\bmu_t$ of the true mean $\bmu^*$ can be found using an upper-confidence-bound approach: if $A_t,\bmu_t$ are defined as in Algorithm~\ref{algo:cov}, then, since $\bmu^*\in\cC_t(A^*)$, we have \[\be_{A_t}\transpose\bmu_t \geq \be_{A^*}\transpose\bmu^*.\]
\vspace{-.3cm}

\begin{algorithm}[t]
\begin{algorithmic}
\STATE \textbf{Initialization}: 
\STATE Play $A_1=[n]$, or at least a sequence $A_1,A_2,\dots,$ (no more than $n(n-1)/2$) such that for any $i,j\in [n]$, one of these $A_{t}$'s contains $\sset{i,j}$. We thus have $\counter{ij}{t-1}\geq 1$ for all $i,j\in [n]$.
\STATE \textbf{For all subsequent rounds} $t$:
\STATE Solve the following bilinear program to get $A_t$, with $\cC_t(A)$ defined by \eqref{C_t(A)}, and play $A_t,$
\vspace{-0.3cm}
\STATE \quad  \[(A_t,\bmu_t)\in \argmax_{A\in \actions,~\bmu\in\cC_t(A) }\be_A\transpose\bmu.\]
\end{algorithmic}
\caption{\textsc{escb-c} \textit{(Efficient Sampling for Combinatorial Bandits with Covariance estimate})}\label{algo:cov}
\vspace{-0.2cm}
\end{algorithm}

 
\noindent
The regret bound for \textsc{escb-c} is stated in Theorem~\ref{thm:upper_cov} with proof in Appendix~\ref{app:upper_cov}.
\begin{theorem}\label{thm:upper_cov}
Assume that the outcome distribution $\Prb_{\bX}$ satisfies Assumptions~\ref{ass:subgau} and~\ref{ass:subexp}, and define $\Delta\triangleq \min_{A\in \actions,~\Delta(A)>0} \Delta(A)$, $\Delta_{\max}\triangleq \max_{A\in \actions,~\Delta(A)>0} \Delta(A)$. If $\Delta$ is small enough, i.e., there exists a universal constant $c$ such that 
\vspace{-.2cm}
$$
\Delta\vee\pa{\Delta{+\Delta\log\pa{\frac{\Delta_{\max}}{\Delta}}}}^{3/2} \leq c\pa{\frac{\log(m+1)\sum_{i\in [\arms]}\max_{A\in \actions,i\in A}\sum_{j\in A}0\vee\Sigma^*_{ij}}{n^2}}^{3/2},
$$
 then the regret of Algorithm~\ref{algo:cov} is upper bounded as \[\limsup_{T\to \infty}\! \frac{\Delta}{\log T} R_T \leq c'  \log^2(m+1)\!\sum_{i\in [\arms]}\max_{A\in \actions,i\in A}\sum_{j\in A}0\vee\Sigma^*_{ij},\]
 \vspace{-.4cm}
 
\noindent
where $c'$ is a universal constant.
\end{theorem}
Notice that the  bound in Theorem~\ref{thm:upper_cov} is
tight, up to a poly-logarithmic factor in $m$, with respect to the lower bound  in Theorem~\ref{thm:lower}, in the case where $\bSigma^*$ has non-negative entries.
Moreover,  we focus on the asymptotic behavior of the regret (w.r.t.\  $T$) when $\Delta$ is small, i.e., when the problem becomes very difficult. While the quantity $c  \log^2(m+1) \sum_{i\in [\arms]}\max_{A\in \actions,i\in A}\sum_{j\in A}0\vee\Sigma^*_{ij} \log(T)/\Delta$ presented in Theorem~\ref{thm:upper_cov} highlights the main dependence on both $\Delta$ and $T$, 
we prove a more precise non-asymptotic upper bound in Appendix~\ref{app:upper_cov}, \eqref{rel:fullstate_cov}, which holds for all $\Delta>0$. Indeed, as for $\textsc{ucb-v}$, the errors from estimating $\bold\Sigma$ generate an extra term in the upper bound. However, since these errors are multiplied with estimation errors on the means, their impact is of second order. In particular, 
 for $\Delta$ small enough, this extra term becomes negligible compared to the main term. Therefore, the term from covariance estimation errors is \textit{not} present in Theorem~\ref{thm:upper_cov},
but appears when $\Delta$ is  far from~$0$. 
Finally, remark that when the covariance $\bSigma^*$ is known, then one can consider the confidence region where $\Sigma_{ij,t}$ is replaced by $\Sigma^*_{ij}$. This avoids covariance estimation errors, and gives the upper bound of Theorem~\ref{thm:upper_cov} when $\Delta{+\Delta\log\pa{{\Delta_{\max}}/{\Delta}}}$ is smaller than $\displaystyle {\sum_{i\in [\arms]}\max_{A\in \actions,i\in A}\sum_{j\in A}0\vee\frac{\Sigma^*_{ij}}{n\cdot m}}\cdot$
\vspace{-.15cm}

\begin{rem}\label{rk:intersect}
Considering the intersection of the region from Algorithm~\ref{algo:cov} with the one of \textsc{cucb-v}, we can replace $\log^2(m+1)\sum_{j\in A}0\vee\Sigma^*_{ij}$ by ${m\Sigma^*_{ii}}\wedge\pa{\log^2(m+1)\sum_{j\in A}0\vee\Sigma^*_{ij}}$
 in Theorem~\ref{thm:upper_cov}. \end{rem}
 \vspace{-.5cm}


\section{Application to sparse outcomes}
In this section, we shall consider an additional structural assumption on the vector $\bX$, namely that it is $s$-sparse in the sense that
$$\norm{\bX}_0\leq s,$$
i.e., the number of nonzero components of $\bX$ is smaller than $s$, where $s$ is a fixed known parameter.\footnote{For example, the Dirichlet-multinomial distribution with $s$ trials is $s$-sparse.} Importantly, the set of components which are nonzero is \emph{not fixed nor known}, and may change over time. It should be noted, however, that there is a significant difference between the stochastic  and the adversarial cases: In the later, the set of components which are nonzero change arbitrarily over time, whereas in the former, this set is  sampled i.i.d. Notice, this sparse stochastic setting is different than the usual stochastic sparse bandit, where $\bmu^*$ is assumed to be sparse; see e.g., \citet{KwonPV17} for the classical MAB setting, and \citet{yadkori12sparse,carpentier12sparse} for the linear bandit setting.  For simplicity, we further assume that $\norm{\bX}_\infty \leq 1$. As we already saw in Proposition~\ref{prop:bounded}, this implies Assumption~\ref{ass:compo_subgau}~and~\ref{ass:subexp}.
The difficulty of this setting is that both the approach of \citet{Degenne2016} and standard methods such as \textsc{cucb-v} would not reach the lower bound for the regime $s\leq m$, as we will see. The reason is that a correlation exists between the components, because of sparsity, and must be taken into account.

\paragraph{Why sparsity in semi-bandits?}
Sparsity is nowadays a very standard assumption in learning theory (that potentially does not need any further motivations). There are many examples of online learning scenarios  naturally involving some sparse structure. For instance, in the celebrated click-through-rate optimization, it is safe to assume that users would only click on a few of the different ads that can be displayed (those that can catch their eyes for any reason, say). 
Similarly, in recommender systems, it is safe to assume that a user will browse/buy items from a specific category and not the other (for instance, a segment of the population in e-shops only buy bottles of wines and others only video-games or clothes).

Other examples involve settings where outcomes are usually zero except on very rare occasions: In the online routing, the packets are sent in a network and are lost if a server of that network has a failure. Because of failsafe procedures, failures are \textit{desynchronized} and typically only one (or at most a few) of them can happen simultaneously. 
In all of these examples, the decision maker has some combinatorial problem to solve: select an admissible path, select a \textit{diverse} bundle of object/ads to display, etc., and only a few of the base items will generate non-zero outcome.

\subsection{Lower bound}
To start our study of sparse outcomes,  
we state a new lower bound in Theorem~\ref{thm:lower_sparse}, that is  valid for the setting  described above. This lower bound is built on the same ideas as Theorem~\ref{thm:lower}, with a notable variation: when reducing to a MAB problem, we do not obtain the necessary conditions for the application of \citet{lai1985asymptotically}, because of the linear dependence between the $\mu_i^*$'s.
Thus, we use instead the lower bound from \citet{graves1997asymptotically}.
More precisely, we consider the same action space $\cA$, and  incorporate the sparsity assumption as an extra constraint for defining a worst case distribution.
\begin{theorem}\label{thm:lower_sparse}
 For any $n,m,s\in \N^*$ such that $n/m$, $n/s$, $1\vee (s/m)$ are integers, $n/m,n/s\geq 2$, any $\Delta\in (0,\frac{ms}{2(n-m)}]$ and any consistent policy, there is a problem with $n$ arms --- characterized by some action space $\actions$ with $m=\max\sset{\abs{A},~A\in \actions}$ and some vector of outcomes $\bX$ with all gaps equal to $\Delta$ satisfying  $\norm{\bX}_\infty \leq 1,~\norm{\bX}_0\leq s$ --- on which the regret satisfies
 
 \[\liminf_{T\to \infty} \frac{\Delta }{\log\pa{T}}R_T\geq\frac{s(s\wedge m)\pa{1-2m/\arms}}{4}\cdot\]
 \vspace{.01cm}
 
\end{theorem}
\vspace{-.25cm}

\noindent
The proof is given in Appendix~\ref{app:lower_sparse}. To give an idea, contrary to Theorem~\ref{thm:lower}, we have more freedom in the covariance, and ${\bold X}$ can be chosen to maximize $\mathbb{V}(\sum_{i\in A}X_i)$ for each action $A$, up to  the constraints ${\Vert{\bold X}\Vert_\infty\leq 1,~\Vert{\bold X}\Vert_0= s}$. The maximal value of $\sum_{i\in A}X_i$ is thus $(s\wedge m)$.
 Now consider for simplicity the softer constraint $\mathbb{E}\Vert{\bold X}\Vert_0= s$. If ${\bold X}$ is chosen so that $\sum_{i\in A}X_i/(s\wedge m)$ is Bernoulli of parameter $p$, then the optimal $p$ is equal to 
 $(s\vee m)/\arms$. 
 The variance is about $p(s\wedge m)^2=ms(s\wedge m)/\arms.$ Multiplying this by $n/m$ (the number of actions) and dividing by the gap $\Delta$ gives the order of the lower bound.
\vspace{-.1cm}
\subsection{Our approach for sparse semi-bandits}

In this subsection, we adapt our techniques to the sparse semi-bandit setting.
Since $\Vert{\bold X}\Vert_\infty\leq 1$, the $\ell_0$-inequality  $\Vert{\bold X}\Vert_0\leq s$ immediately implies the $\ell_1$-inequality  $\Vert{\bold X}\Vert_1\leq s$. As we will actually only use sparsity  through the latter inequality, we can relax our assumption on the model into  $\norm{\bX}_1\leq s$, for more generality. 
Let $\nu^*_i\triangleq\EE{\abs{X_i}}$, and $\meanabs{i,t-1}$ the corresponding empirical average estimate:
\(\meanabs{i,t-1}\triangleq\frac{\sum_{u\in[t-1]}\II{i\in A_u}\abs{X_{i,u}}}{\counter{i}{t-1}}\cdot\)
Our approach is based on replacing $\sum_{j\in A} 0\vee\Sigma^*_{ij}$ by $\nu^*_i(s\wedge m)$ (see Lemma~\ref{lem:sparse}, proved in Appendix~\ref{app:sparse}).  Using this, it is possible to estimate $\nu^*_i$ instead of each $\Sigma^*_{ij}$.
\begin{lem}\label{lem:sparse}  $\sum_{j\in A}0\vee\Sigma^*_{ij}\leq 2\nu^*_i(s\wedge m)$.
\end{lem}
We can therefore use the same algorithm (Algorithm~\ref{algo:cov}), but with a confidence region $\cC_t$ independent of $A$, since summing over $A$ or $[n]$ on the main sum doesn't change the algorithm and the second sum $\sum_{j\in A}0\vee\Sigma_{ij,t}$ is replaced by an estimates of the upper bound given in Lemma~\ref{lem:sparse}.
\vspace{-0.55cm}

\begin{align}
\cC_t\triangleq \vmean{t-1}\!+\!\left\{\!\sizecorr{\sum_{i\in A}}\bxi\in \R^\arms,\!
\sum_{i\in [\arms]} \frac{\counter{i}{t-1}\xi_i^2}{{m\abs{\xi_i}\!+\!2(s\!\wedge\! m)\nu_{i,t} }}\!\leq 8( \log(t)\!+\!\log(\log(t)) )\!+\!4em\!\right\},\label{C_t}\end{align}
\vspace{-0.15cm}

where the upper bound estimate $\nu_{i,t}\triangleq \meanabs{i,t-1}\!+\!\sqrt{\frac{1.5\log(t)}{\counter{i}{t-1}}}$ of $\nu^*_i$ is a simple consequence of Hoeffding's inequality, using that $\abs{X_{i,u}}$ is $1/4$-sub-Gaussian. 
Our algorithm is stated in Algorithm~\ref{algo:sparse}. As a byproduct of Theorem~\ref{thm:upper_cov}, we provide an upper bound for the regret in the sparse semi-bandit setting in Corollary~\ref{cor:upper_cov} (see Appendix~\ref{app:upper_cov2}, \eqref{rel:fullstate_sparse}, for a more precise bound). Again, notice we are reaching the lower bound of Theorem~\ref{thm:lower_sparse}, using the relation $\sum_i \nu^*_i =\mathbb{E}\Vert{\bold X}\Vert_1\leq s.$

\begin{algorithm}[t]
\begin{algorithmic}
\STATE \textbf{Initialization}: 
\STATE Play $A_1=[n]$, or at least a sequence $A_1,A_2,\dots,$ (no more than $n$) such that
all arm have been sampled once.
 We thus have $\counter{i}{t-1}\geq 1$ for every arm $i\in [n]$.
\STATE \textbf{For all subsequent rounds} $t$:
\STATE Solve the following bilinear program to get $A_t$, with $\cC_t$ defined by \eqref{C_t}, and play $A_t$.
\vspace{-0.3cm}
\STATE \quad  \[(A_t,\bmu_t)\in \argmax_{A\in \actions,~\bmu\in\cC_t }\be_A\transpose\bmu.\]
\end{algorithmic}
\caption{\textsc{escb-c} modified for the case of $\norm{\cdot}_1$-constrained outcomes}\label{algo:sparse}
\vspace{-0.2cm}
\end{algorithm}

\vspace{-.1cm}
\begin{cor}\label{cor:upper_cov}Assume that the outcome distribution $\Prb_{\bX}$ satisfies $\Vert{\bold X}\Vert_\infty\leq 1$ and $\Vert{\bold X}\Vert_1\leq s$, and that
\vspace{-0.26cm}$$\pa{\Delta(s\wedge m)}^{2/3}\vee\pa{m\Delta+m\Delta\log(\Delta_{\max}/\Delta)}\leq c\log(m+1){\sum_{i\in [\arms]}\frac{\nu^*_i(s\wedge m)}{n}}\CommaBin$$
\vspace{-0.4cm}
for some universal constant $c$. Then the regret of Algorithm~\ref{algo:sparse} is upper bounded as 

\begin{align*}\limsup_{T\to \infty} \frac{\Delta}{\log\pa{T}}R_T \leq c'  \log^2(m+1) \sum_{i\in [n]}\nu^*_i(s\wedge m)\leq c'  \log^2(m+1) (s\wedge m) s,\end{align*}
\vspace{-0.35cm}

\noindent
where $c'$ is a universal constant.
\end{cor}

\begin{rem}
It should be noticed that semi-bandits algorithms as \textsc{cucb-v} or \textsc{cucb-kl} (that are variant of the classical \textsc{cucb} \citep{kveton2015tight}, where the confidence region is a Cartesian product of confidence intervals, with Bernstein and kl-base confidence intervals respectively) also reach the lower bound of Theorem~\ref{thm:lower_sparse} for the regime $s\geq m$, since $\mathbb{V}(X_i)\leq2\nu^*_i$ (thanks to Lemma~\ref{lem:sparse}). However, in the regime where $s\leq m$, these algorithms are not able to reach it, while \textsc{escb-c} is. In Appendix~\ref{app:cucb}, we describe the two algorithms \textsc{cucb-v} and \textsc{cucb-kl}, and comment further on the tightness difference between confidence regions.
\end{rem}
\section{Implementation details}
\label{sec:imp}
We now discuss the computational efficiency of our approaches.
First, Algorithm~\ref{algo:cov} (and both those of \citet{combes2015combinatorial} and \citet{Degenne2016}) is not efficient for arbitrary combinatorial  space $\mathcal{A}$. 
However, the evaluation of  \(F:A\mapsto\max_{{\boldsymbol \mu}\in \mathcal{C}_t(A)} {\bold e}_A\transpose{\boldsymbol \mu}, \)
can be done efficiently as it is an LP over a convex set. In practice, when $\mathcal{A}$ allows it, \textsc{greedy}\footnote{Starting from $A=\emptyset$, we sequentially add (when possible) the best possible $i$ to the current $A$ if $F(A\cup\sset{i})>F(A)$.} \citep{Nemhauser1978} can be used to maximize $F$. In general, it is unknown if this alters the regret rate. On the one hand, it does not when~$\mathcal{A}$ is given by a matroid, and  $\mathcal{C}_t$ is as in Algorithm~\ref{algo:sparse}. This is  because $F$ is \emph{submodular} and
the following approximation guaranty holds for the output $A_t$ of \textsc{greedy} \citep{perrault2019exploiting}:
\(2\pa{F(A_t) - \be_{A_t}\transpose\vmean{t-1}}+\be_{A_t}\transpose\vmean{t-1}\geq F(A^*),\)
where the l.h.s.\ is simply $F$ where $\cC_t$ is scaled by a factor $2$ from its center $\vmean{t-1}$. 
On the other hand, when $\mathcal{C}_t(A)$ is as in Algorithm~\ref{algo:cov}, a concave extension of $A\mapsto F(A)$ can be considered, and can thus be maximized efficiently.
Notice, when considering the intersection of the two confidence regions as in Remark~\ref{rk:intersect}, this implementation is still tractable since the minimum of two concave functions is still concave.
Since the obtained solution might not be integral, we use a randomized rounding to obtain a feasible set $A_t\in\cA=\sset{0,1}^n$. We provide in Appendix~\ref{app:super} further details and prove that this method scales the regret by a factor $1+\log\pa{\frac{m\log(T)}{\Delta^2}}$,  an acceptable price for efficiency.   
\section{Experiments}
\begin{figure}[H]
\centering
\vspace{-0.75cm}
\includegraphics[width=.38\columnwidth]{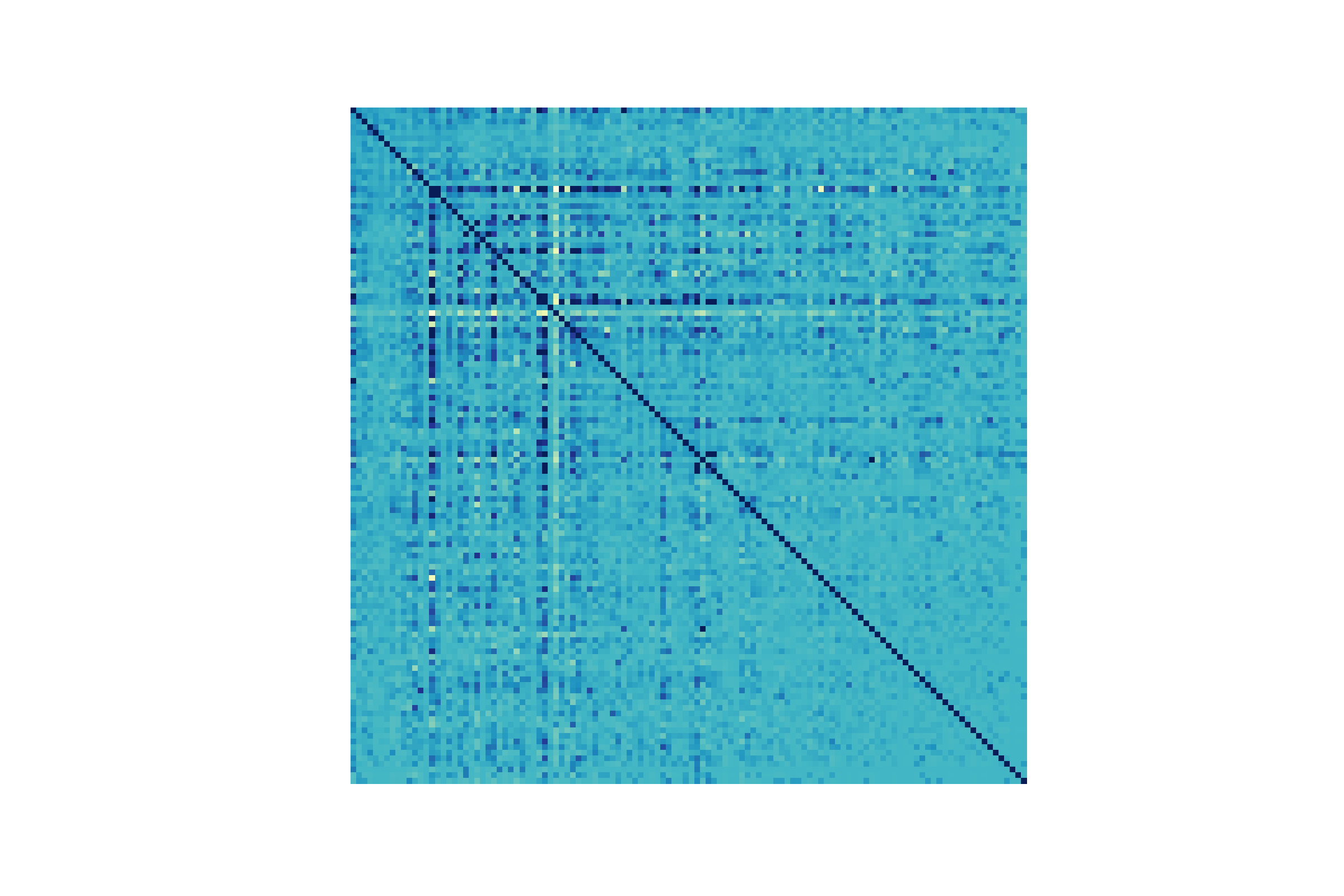}\hfill\includegraphics[width=.25\columnwidth]{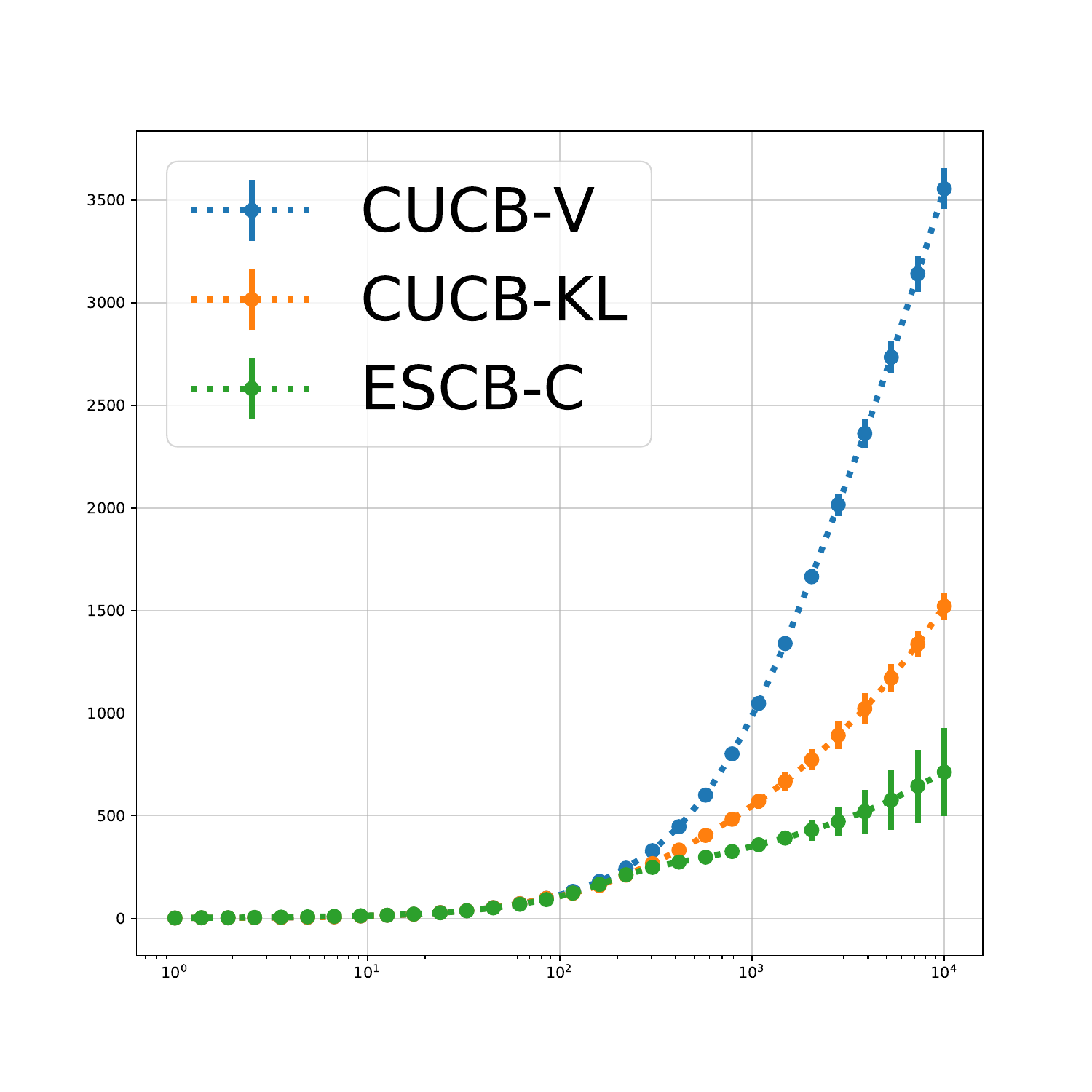}
 \vspace{-.7cm}
\caption{\textbf{Left}: Correlation matrix of the dataset, \textbf{right}: Cumulative regret, averaged over $36$ independent simulations}
\label{exp:cor}
\end{figure}
\vspace{-.65cm}

We consider the following dynamic assortment problem. An agent has $\arms$ products to sale, with fixed known prices. At each round, a customer arrives, with some unknown random valuation vector over products. Then, the agent offers any subset of products, by paying a fixed known cost for each offered product (e.g. transportation and display cost), and the customer buys an offered product if and only if its valuation is greater than its price. The agent is interested in maximizing the total profit (revenue minus cost) from sales over $T$ rounds. We use the $n=120$ products from the \citet{kaggle2013} dataset
containing $7500$ grocery store transactions. At each round, valuations are determined 
by sampling a random transaction from this dataset.  The choice of such data is motivated by correlations that exist between arms, as illustrated in Figure~\ref{exp:cor} -- left, representing the correlation matrix. We ran $36$ independent simulations with $T=10^4$, and with a common product price and cost respectively   equal to $1.5$ and  $0.1$. We compared \textsc{cucb-v} and \textsc{cucb-kl} (see Appendix~\ref{app:cucb}) with the Lov\'asz extension implementation of \textsc{escb-c} (see Appendix~\ref{app:super}) and results are plotted in log-scale (Figure~\ref{exp:cor} -- right); error bars represent the sample standard deviation over simulations. There is less volatility in the regret of \textsc{cucb-v} and \textsc{cucb-kl}; this is due to the fact that their confidence regions overestimate the risk, and the ``bad'' event where the regret deviates is almost negligible. Nevertheless, we clearly observe that \textsc{escb-c} outperforms the two other approaches in terms of the average regret. 
Finally, let us point out that we did not empirically compare to the \textsc{ols-ucb} algorithm of \citet{Degenne2016} since it is inefficient to implement (the combinatorial problem to be solved within each round is NP-Hard in general \citep{atamturk2017maximizing}. We noticed that for the choice of sub-Gaussian matrix where all the correlation coefficients equals 1, \textsc{ols-ucb} (if it could be implemented) would return a solution very close to \textsc{cucb-v}.

\vspace{-0.2cm}
\section{Discussion}
We improved the analysis of combinatorial semi-bandits in multiple ways. First, we brought new perspectives by considering a fairly large family of sub-exponential probability distributions, that crucially  do not depend on parameters difficult to obtain in real situations. We have built an algorithm for this family,  based on the estimation of the covariance matrix. We have therefore already significantly improved existing approaches by adapting not only to the variance of the arms, but also to the correlation  between them.
A tight analysis of our proposed method gives a new state-of-the-art upper bound on the regret. Our new bound is also more intuitive, and is more relevant to  reflect the complexity of the instance at hand (through  correlations  between  arms). Finally, we applied our approach to a setting not yet studied before, that assumes sparsity of the outcome vector. We gave a lower bound, as well as a matching algorithm that  leverages the sparsity assumption. 
\clearpage

\acks{The research presented was supported by European CHIST-ERA project DELTA, French Ministry of
Higher Education and Research, Nord-Pas-de-Calais Regional Council,  French National Research Agency project BOLD (ANR19-CE23-0026-04). Furthermore, it was also supported in part by a public grant as part of the Investissement d'avenir project, reference ANR-11-LABX-0056-LMH, LabEx LMH, in a joint call with Gaspard Monge Program for optimization, operations research and their interactions with data sciences.}
\setlength{\bibsep}{5.1pt}
\bibliography{library}
\newpage

\appendix
\section{The sub-Gaussian matrix is an upper bound on the covariance matrix}
\label{app:bSigmapreceqbGamma}
The fact that $\bSigma^*\preceq\bC$ is well known and can be proved as follows: Fix $\bx\in \R^\arms$. For any $ \lambda~\in~\R,$ $\EE{e^{\lambda\bx\transpose\pa{\bX-\bmu^*}}}~\leq~e^{\lambda^2\bx\transpose\bC\bx/2}$. The second order Taylor expansion in $\lambda$ gives
$$\frac{\lambda^2}{2}\EE{\pa{\bx\transpose\pa{\bX-\bmu^*}}^2}+o\pa{\lambda^2}\leq \frac{\lambda^2}{2}\bx\transpose\bC\bx +o\pa{\lambda^2}.$$
Dividing the inequality by $\lambda^2$, and letting $\lambda\to 0$ yields $\EE{\pa{\bx\transpose\pa{\bX-\bmu^*}}^2}\leq \bx\transpose\bC\bx$, i.e., $\bSigma^*\preceq \bC$.

\section{Proof of Proposition~\ref{prop:bounded}}
\label{app:bounded}
\begin{proof}
Assumption~\ref{ass:compo_subgau} is a direct consequence of  Hoeffding's Lemma. For Assumption~\ref{ass:subexp}, we have $\norm{\bX-\bmu^*}_\infty\leq 2\kappa$. For $\norm{\blambda}_1\leq1/(2\kappa)$, we have:
\begin{align*}
 \log\EE{e^{\blambda\transpose\pa{\bX-\bmu^*}}}&=\log\pa{1+\sum_{k\geq 2}\EE{\frac{\pa{\blambda\transpose\pa{\bX-\bmu^*}}^k}{k!}}}\\&\leq
 \sum_{k\geq 2}\EE{\frac{\pa{\blambda\transpose\pa{\bX-\bmu^*}}^k}{k!}}&\log(x)\leq x-1~\forall x> 0,\\&=
 \sum_{k\geq 2}\EE{\frac{\pa{\blambda\transpose\pa{\bX-\bmu^*}}^{k-2}\pa{\blambda\transpose{\pa{\bX-\bmu^*}}}^2}{k!}}\\
 &\leq  \sum_{k\geq 2}\EE{\frac{\pa{\norm{\blambda}_1\norm{\bX-\bmu^*}_\infty}^{k-2}\pa{\blambda\transpose{\pa{\bX-\bmu^*}}}^2}{k!}}
 \\
 &\leq  \sum_{k\geq 2}{\frac{\EE{\pa{\blambda\transpose{\pa{\bX-\bmu^*}}}^2}}{k!}}
 \\&=(e-2)\blambda\transpose\bSigma^*\blambda\leq \blambda\transpose\bSigma^*\blambda.
\end{align*}\end{proof}

\section{Proof of Theorem~\ref{thm:lower}}
\label{app:lower}
\begin{proof}
Consider $\actions$ containing $\arms/m$ disjoint actions $A_1,\dots,A_{\arms/m}$ composed of $m$ arms, with $A_k=\sset{(k-1)m+1,\dots,km }$, and $\bX\sim \cN(-\Delta/m(\II{i\notin A_1})_i,\bSigma^*)$. 
 This problem  reduces to a standard bandit problem with $\arms/m$ arms.
 We use a result from \citet{burnetas1996optimal}, a generalization of \citet{lai1985asymptotically}, that states that
 \[\liminf_{T\to \infty} \frac{R_T}{\log\pa{T}}\geq \sum_{k=2}^{\arms/m}\frac{\Delta}{\inf_{Y,~\EE{Y}=0}\KL{\Prb_{\sum_{i\in A_k} X_i}}{\Prb_{Y}}}\cdot\]
As we can write \begin{align*}\inf_{Y,~\EE{Y}=0}\KL{\Prb_{\sum_{i\in A_k} X_i}}{\Prb_{Y}} &\leq {\KL{\cN\pa{-\Delta,\be_{A_k}\transpose \bSigma^*\be_{A_k} }}{\cN\pa{0,\be_{A_k}\transpose \bSigma^*\be_{A_k}}}} \\& = \frac{\Delta^2/2}{\be_{A_k}\transpose \bSigma^*\be_{A_k}},\end{align*}
it holds that
$$\liminf_{T\to \infty} \frac{R_T}{\log\pa{T}}\geq 2\sum_{k=2}^{n/m}\frac{\be_{A_k}\transpose \bSigma^*\be_{A_k}}{\Delta} = 2\sum_{i\in[n],~i\notin A_1}\max_{A\in \cA,~i\in A}\sum_{j\in A}\frac{\Sigma^*_{ij}}{\Delta},$$ where we used the fact that $\sset{A\in \cA,~i\in A}$ is a singleton.
\end{proof}

\section{Proof of Proposition~\ref{prop:sigmaupper}}
\label{app:sigmaupper}
\begin{proof}We define $\tilde\Sigma_{ij,t-1}\triangleq\sum_{u\in[t-1]}\frac{\II{i,j\in A_u}\pa{X_{i,u}-\mu_i^*}\pa{X_{j,u}-\mu_j^* }}{\counter{ij}{t-1}}$ and
for $k\in \sset{i,j}$, $\tilde\mu_{k,t-1}\triangleq\frac{1}{\counter{ij}{t-1}}\sum_{u\in[t-1]}\II{i,j\in A_u}X_{k,u}$.
Notice that the following relation holds
 $$\bar\Sigma_{ij,t-1}=\tilde\Sigma_{ij,t-1}+\pa{\mu_i^*-\mean{i,t-1}}\pa{\tilde\mu_{j,t-1}-\mean{j,t-1}}+\pa{\mu_j^*-\mean{j,t-1}}\pa{\tilde\mu_{i,t-1}-\mu_i^*}.$$
 We now state Lemma~\ref{lem:subex} giving sub-exponential parameters for a product of sub-Gaussian random variables. A proof comes from \citet{honorio14}.
 \begin{lem}\label{lem:subex}
 $Y$,$Z$ are $1$-sub-Gaussian random variables $\imp$ $\forall\abs{\lambda}\leq 1/8$, $\EE{e^{\lambda(YZ-\EE{YZ})}}\leq e^{64\lambda^2}$.
 \end{lem}
 We apply Lemma~\ref{lem:subex} with a Chernoff argument and an union bound (to avoid the randomness of counters) in order to get the following Bernstein inequality $$\PP{ \abs{\Sigma^*_{ij}-  \tilde\Sigma_{ij,t-1}}\geq16\pa{\frac{3\log(t)}{\counter{ij}{t-1}}\vee\sqrt{\frac{3\log(t)}{\counter{ij}{t-1}}}} }\leq 2t^{-2}.$$
 In the same way, Hoeffding's inequality gives directly that with probability $1-8t^{-2}$, we have simultaneously
 \begin{equation*}
\left\lbrace
\begin{array}{cc}
\abs{\mu_i^*-\mean{i,t-1}} &\leq \sqrt{\frac{6\log(t)}{\counter{i}{t-1}}}\\
\abs{\tilde\mu_{j,t-1}-\mean{j,t-1}} &\leq \sqrt{\frac{8\log(t)}{\counter{ij}{t-1}}}\\
\abs{\mu_j^*-\mean{j,t-1}}&\leq \sqrt{\frac{6\log(t)}{\counter{j}{t-1}}}\\
\abs{\tilde\mu_{i,t-1}-\mu_i^*}&\leq \sqrt{\frac{6\log(t)}{\counter{ij}{t-1}}},
\end{array}\right.
\end{equation*}
which is enough to conclude the proof. Notice that for the second inequality above, we take the union bound for two counters. When they are not random, $\counter{ij}{t-1}\pa{\tilde\mu_{j,t-1}-\mean{j,t-1}}$, that is equal to
 $$\sum_{u\in[t-1]}\II{i,j\in A_u}X_{j,u}\pa{1-\counter{ij}{t-1}/\counter{j}{t-1}} - \sum_{u\in[t-1]}\II{j\in A_u,i\notin A_u}X_{j,u} \counter{ij}{t-1}/\counter{j}{t-1},$$
 is a sum of $\counter{j}{t-1}$ independent random variables, $\counter{ij}{t-1}$ of which are $\pa{1-\counter{ij}{t-1}/\counter{j}{t-1}}^2$-sub-Gaussian and the remaining ones are $\counter{ij}{t-1}^2/\counter{j}{t-1}^2$-sub-Gaussian. So it is $\counter{ij}{t-1}\pa{1-\counter{ij}{t-1}/\counter{j}{t-1}}$-sub-Gaussian, and in particular $\counter{ij}{t-1}$ -sub-Gaussian.
\end{proof}

\section{Proof of Theorem~\ref{thm:upper_cov}}
\label{app:upper_cov}
\begin{proof}
In the proof, we denote $\ba\odot\bb\triangleq (a_ib_i)_i$ the Hadamard product of two vectors $\ba,\bb\in \R^\arms.$
Let $t\geq 1$, and $\delta(t)\triangleq 2(\log(t)+\log(\log(t)))+em$.
Through initialization, we can assume $\counter{ij}{t-1}\geq 1$ for all $i,j\in [\arms]$ (as this only adds $n(n-1)\Delta_{\max}/2$ to the regret bound).
We will decompose contributions to regret by considering the following events:
\[\mathfrak{C}_{t}\triangleq \sset{\bmu^*\vee \bar\bmu_{t-1} \in \cC_t\pa{A^*}},\]
\[\mathfrak{D}_{t}\triangleq \sset{\be_{A_t}\transpose \pa{\vmean{t-1}-\bmu^*}\leq  \Delta\pa{A_t}/2},\]
\[\mathfrak{S}_{t}\triangleq \sset{\forall i,j\in [\arms], 0\leq \Sigma_{ij,t}-\Sigma^*_{ij}\leq  2g_{ij}(t)}.\]
We also define
\[\tilde g_{ij}(t)\triangleq 
16\pa{\frac{3\log(t)}{\counter{ij}{t-1}^2}\vee\sqrt{\frac{3\log(t)}{\counter{ij}{t-1}^3}}} + \sqrt{\frac{48\log^2(t)}{{\counter{ij}{t-1}^4}}}+\sqrt{\frac{36\log^2(t)}{{\counter{ij}{t-1}^4}}}\CommaBin\]
 \[\forall i\in[\arms],~\Delta_{i,\min}\triangleq \min_{A\in \actions,~i\in A,~\Delta\pa{A}>0}\Delta\pa{A},\]
  \[\Delta_{i,\max}\triangleq \max_{A\in \actions,~i\in A,~\Delta\pa{A}>0}\Delta\pa{A},\]
 and \[\forall i,j\in[\arms],~\Delta_{ij,\min}\triangleq \min_{A\in \actions,~i,j\in A,~\Delta\pa{A}>0}\Delta\pa{A},\]
 \[\Delta_{ij,\max}\triangleq \max_{A\in \actions,~i,j\in A,~\Delta\pa{A}>0}\Delta\pa{A}.\]
\paragraph{Step 1: If $\mathfrak{C}_{t},\mathfrak{D}_{t}$ and $\mathfrak{S}_{t}$ hold}
We have
\begin{align*}
  \Delta\pa{A_t}&= \pa{ \be_{A^*}-\be_{A_t}}\transpose{\bmu^*}\\&\leq\be_{A^*}\transpose{\bmu^*\vee \bar\bmu_{t-1}} - \be_{A_t}\transpose\bmu_t + \be_{A_t}\transpose\pa{\bmu_t - \bmu^*}
   \\&\leq
   \be_{A_t}\transpose\pa{\bmu_t - \bmu^*} & \mathfrak{C}_{t}
      \\&\leq
   \Delta\pa{A_t}/2 + \be_{A_t}\transpose\pa{\bmu_t - \vmean{t-1}} & \mathfrak{D}_{t}
\end{align*}
i.e.,
\begin{align*}
 \Delta\pa{A_t} &\leq 2\be_{A_t}\transpose\pa{\bmu_t - \vmean{t-1}}
 \\&\leq
 2\sqrt{\sum_{i\in A_t}\frac{4\pa{ \sum_{j\in A_t}0\vee\Sigma_{ij,t}+m\pa{\mu_{i,t}-\mean{i,t-1}}} {\delta(t)}}{\counter{i}{t-1}}}&\text{Cauchy-Schwarz and }\bmu_t\in C_t(A_t)
 \\&\leq
 4\sqrt{\delta(t)\sum_{i\in A_t}\pa{\frac{ \sum_{j\in A_t}0\vee\Sigma_{ij,t}} {\counter{i}{t-1}}+\frac{m\pa{\mu_{i,t}-\mean{i,t-1}}}{\min_{j\in A_t}\counter{j}{t-1}}}}.
\end{align*}
Solving the corresponding quadratic inequation in the variable $x=\be_{A_t}\transpose\pa{\bmu_t - \vmean{t-1}}$, we get
\begin{align*}
\Delta\pa{A_t} &\leq 2\be_{A_t}\transpose\pa{\bmu_t - \vmean{t-1}}
 \\&\leq 4\pa{\sqrt{\frac{\delta(t)^2m^2}{{\min_{i\in A_t}\counter{i}{t-1}^2}} + \sum_{i\in A_t}\frac{\delta(t) \sum_{j\in A_t}0\vee\Sigma_{ij,t}} {\counter{i}{t-1}}}+\frac{m\delta(t)}{{\min_{j\in A_t}\counter{j}{t-1}}}}
 \\&\leq 4{ \sqrt{\delta(t)\sum_{i\in A_t}\frac{ \sum_{j\in A_t}0\vee\Sigma_{ij,t}} {\counter{i}{t-1}}}+\frac{8m\delta(t)}{{\min_{j\in A_t}\counter{j}{t-1}}}}
 \\&\leq
 4{ \sqrt{\delta(t)\sum_{i\in A_t}\frac{ \sum_{j\in A_t}0\vee\pa{\Sigma^*_{ij}+ 2g_{ij}(t)}} {\counter{i}{t-1}}}+\frac{8m\delta(t)}{{\min_{j\in A_t}\counter{j}{t-1}}}}&\mathfrak{S}_{t}
\\&\leq
 4 \sqrt{\delta(t)\sum_{i\in A_t}\frac{ \sum_{j\in A_t}0\vee\Sigma^*_{ij}}  {\counter{i}{t-1}}}
 +
 4 \sqrt{\delta(t)\sum_{i\in A_t}\frac{ \sum_{j\in A_t} 2g_{ij}(t)}  {\counter{i}{t-1}}}
 +
 \frac{8m\delta(t)}{{\min_{j\in A_t}\counter{j}{t-1}}}
 \\&\leq
 \underbrace{4 \sqrt{\delta(T)\sum_{i\in A_t}\frac{ \max_{A\in \actions,i\in A}\sum_{j\in A}0\vee\Sigma^*_{ij}}  {\counter{i}{t-1}}}}_{\numterm{term1}}
 +
 \underbrace{4 \sqrt{\delta(T)\sum_{i,j\in A_t}{  \tilde g_{ij}(T)}  }}_{\numterm{term2}}
 +
  \underbrace{\frac{8m\delta(T)}{{\min_{j\in A_t}\counter{j}{t-1}}}}_{\numterm{term3}}\cdot
\end{align*}
Where the last inequality uses that ${\counter{i}{t-1}\wedge\counter{j}{t-1}}\geq \counter{ij}{t-1}\forall i,j\in [\arms]$.
From this point, we treat each term separately, using the relation
$$\II{\Delta\pa{A_t}\leq
\eqref{term1}+\eqref{term2}+\eqref{term3}}\leq \II{\Delta\pa{A_t}/3\leq \eqref{term1}} + \II{\Delta\pa{A_t}/3\leq \eqref{term2}} +\II{\Delta\pa{A_t}/3\leq \eqref{term3}}. $$
We provide Theorem~\ref{thm:dege} in Appendix~\ref{app:extra_cmab}, that is helpful to bound the regret on each of this 3 events. Indeed, for the first term, applying it with $\beta_{i,T}=12^2\delta(T)\max_{A\in \actions,i\in A}\sum_{j\in A}0\vee\Sigma^*_{ij}$, $\alpha_i=1/2$, and $(I,I_t) = ([n],A_t)$ gives the bound $$\sum_{t=1}^T\II{\Delta\pa{A_t}/3\leq \eqref{term1}}\Delta\pa{A_t}\leq 4608\log_2^2(4\sqrt{m})\sum_{i\in [n]}{{{\frac{\delta(T)\max_{A\in \actions,i\in A}\sum_{j\in A}0\vee\Sigma^*_{ij}}{\Delta_{i,\min}}}}}\cdot$$
The second term can be itself decomposed into two terms, bounding the max by the sum and using $\log(T)\leq \delta(T)$.
$$\eqref{term2}\leq 4\delta(T)\sqrt{\sum_{i,j\in A_t} \pa{54+\sqrt{48}}N_{ij,t-1}^{-2}} + 4\delta(T)^{0.75}\sqrt{16\sqrt{3}\sum_{i,j\in A_t}N_{ij,t-1}^{-1.5}}.$$
Thus, again, it is sufficient to treat each term separately. We also apply Theorem~\ref{thm:dege}, but with $(I,I_t) = ([n]^2,A_t^2)$, taking respectively $\alpha_i=1,\beta_{i,T}=24\sqrt{54+\sqrt{48}}\delta(T)$ and $\alpha_i=0.75,\beta_{i,T}=192~\cdot~{6}^{2/3}\delta(T)$ for each term. This gives 
\begin{align*}\sum_{t=1}^T\II{\Delta\pa{A_t}/3\leq \eqref{term2}}\Delta\pa{A_t}~&\leq 1152\sqrt{6} \log_2(4m)\sum_{i,j\in [n]}\delta(T)\pa{1+\log\pa{\frac{\Delta_{ij,\max}}{\Delta_{ij,\min}}}}\\
&+12288\cdot6^{2/3}\pa{{4^{1/3}-1}}^{-1} \log_2(4m)\sum_{i,j\in [n]}\delta(T)\Delta_{ij,\min}^{-1/3}.\end{align*}
The last term can be analyzed in the same way by first upper bounding it as
$$\eqref{term3}\leq 8m\delta(T)\sqrt{\sum_{i\in A_t}\frac{1}{N_{i,t-1}^2}}\cdot$$
Then, taking $\alpha_i=1,\beta_{i,T}=24m\delta(T)$ in Theorem~\ref{thm:dege} gives
\begin{align*}\sum_{t=1}^T\II{\Delta\pa{A_t}/3\leq \eqref{term3}}\Delta\pa{A_t}\leq 1152\log_2(4\sqrt{m})\sum_{i\in[n]} m\delta(T)\pa{1+\log\pa{\frac{\Delta_{i,\max}}{\Delta_{i,\min}}}}.\end{align*}
This concludes step 1; notice that all subsequent steps will aim to bound the regret by a term independent of $T$, over a certain event. Thus, we can see that the bounds above are the actual contributions to the rate of the regret. To show Theorem~\ref{thm:upper_cov}, we must therefore choose the regime for $\Delta\leq \Delta_{i,\min}$ so that the first term prevails over the others. In other words, we want to have
$$n^2\pa{\Delta^{-1/3}\vee\pa{1+\log(\Delta_{\max}/\Delta)}}\leq c \log(m+1)\sum_{i\in [n]}\frac{\max_{A\in \actions,i\in A}\sum_{j\in A}0\vee\Sigma^*_{ij}}{\Delta}\CommaBin$$
where $c$ is a constant. This gives exactly our condition in Theorem~\ref{thm:upper_cov}.

\paragraph{Step 2: If $\mathfrak{S}_{t}$,$\neg\mathfrak{C}_{t}$ hold}

Let $\sigma_i^2\triangleq {\sum_{j\in A^*}{0\vee\Sigma^*_{ij}}}$ for all arms $i\in [\arms].$
We fix some $\delta\geq e\cdot m$, and define the following events:
\[\mathfrak{A}_{t}\triangleq \sset{\sum_{i\in A^*}\II{\mu_i^*\geq \mean{i,t-1}}\counter{i}{t-1}\frac{\pa{\mu_i^*-\mean{i,t-1}}^{2}}{4\pa{\sigma_i^2+\abs{A^*}\pa{\mu_i^*-\mean{i,t-1}}}}\geq \delta }\]
\[\forall\bd\in \pa{\N^*}^{A^*},\quad\mathfrak{B}_{\bd,t}\triangleq\bigcap_{i\in A^*}\sset{e^{d_i-1}\leq\counter{i}{t-1}<e^{d_i}}.\]
Notice that $\mathfrak{S}_{t},\neg\mathfrak{C}_{t}$ implies $\mathfrak{A}_{t}$ for $\delta=\delta(t)$. 
Since each number of pulls $\counter{i}{t-1}$ for $i\in A^*$ is bounded by $t$, the number of possible $\bd\in \pa{\N^*}^{A^*}$ such that $\PP{\mathfrak{B}_{\bd,t}}>0$ is bounded by $ \log(t)^m$. Thanks to the following Lemma~\ref{lem:concentration}, and an union bound on those possible $\bd\in \pa{\N^*}^{A^*}$, we get 
\[\PP{\mathfrak{A}_{t}}\leq e^{m+1}\pa{\frac{{(\delta-1)\log(t)}}{m}}^m e^{-\delta},\]
so the regret under this event is bounded by a universal constant, since the upper bound above is the term of a convergent series for $\delta=\delta(t)$. Indeed, it rewrites as
$$t^{-2}e^{m+1-em}\pa{\underbrace{\frac{2-\log^{-1}(t)}{m}}_{\leq 2/m}+\underbrace{2\frac{{\log(\log(t))}}{\log(t)}+e\log^{-1}(t)}_{\leq 2e^{e/2-1}}}^m,$$
that is bounded by
$$t^{-2}e\cdot \underbrace{\pa{e^{1-e} \cdot2e^{e/2-1}}^m}_{\leq 1}\underbrace{\pa{\frac{{e^{1-e/2}}}{m}+1}^m}_{\leq e^{e^{1-e/2}}}.$$
\begin{lem}[Covering-argument] Let $\bd\in \pa{\N^*}^{A^*}.$ Then, 
 $\PP{\mathfrak{A}_{t}\cap \mathfrak{B}_{\bd,t}}\leq \pa{\frac{(\delta-1) e}{m}}^me^{1-\delta}.$
\label{lem:concentration}\end{lem}
\begin{proof} We rely on a covering argument. The idea is to get rid of randomness by replacing the empirical mean $\bar\mu_{i,t-1}$ by some non-random value $x_i$.
Let $\bzeta\in \R_+^{A^*}$.
For $i\in A^*,$ we define $x_i(N)$ for $N\in \R_+$ as the unique solution $x\in(-\infty,\mu_i^*]$ of the equation ${N}\frac{\pa{\mu_i^*-x}^{2}}{4(\sigma_i^2+\abs{A^*}\pa{\mu_i^*-x})}= \zeta_i$. 
Notice that for all $i\in A^*$, $x_i$ is non-decreasing since $x\mapsto\frac{\pa{\mu_i^*-x}^{2}}{4(\sigma_i^2+\abs{A^*}\pa{\mu_i^*-x})}$ is decreasing on $(-\infty,\mu_i^*]$. The event

$$\bigcap_{i\in A^* }\sset{\counter{i}{t-1}\frac{\pa{\mu_i^*-\mean{i,t-1}}^{+2}}{4\pa{\sigma_i^2+\abs{A^*}\pa{\mu_i^*-\mean{i,t-1}}}}>  \zeta_i } $$implies $$\bigcap_{i\in A^*}\sset{\mean{i,t-1} \leq x_i(\counter{i}{t-1})}. $$
Under the event $\mathfrak{B}_{\bd,t}$, this implies \begin{align}\bigcap_{i\in A^*}\sset{\mean{i,t-1} \leq x_i(e^{d_i})}. \label{ass:x_i}\end{align}
With $\epsilon_i\triangleq\mu_i^*-x_i(e^{d_i})$ and $\lambda_i\triangleq\frac{\epsilon_i}{2\pa{\sigma_i^2+{\abs{A^*}\varepsilon_i}}}\CommaBin$  $i\in A^*$, this further implies:
\begin{align*}e^{-1}\sum_{i\in A^*}\zeta_i&={\sum_{i\in A^*}e^{d_i-1}\frac{\epsilon_i^2}{4\pa{\sigma_i^2+{\abs{A^*}\varepsilon_i}}}}&x_i(e^{d_i})>-\infty,\\&\leq
{\sum_{i\in A^*}\counter{i}{t-1}\frac{\epsilon_i^2}{4\pa{\sigma_i^2+{\abs{A^*}\varepsilon_i}}}}&\mathfrak{B}_{\bd,t}\\&
={\sum_{i\in A^*}\counter{i}{t-1}\frac{\epsilon_i^2}{2\pa{\sigma_i^2+{\abs{A^*}\varepsilon_i}}}}-{\sum_{i\in A^*}\counter{i}{t-1}\frac{\epsilon_i^2}{4\pa{\sigma_i^2+{\abs{A^*}\varepsilon_i}}}}\\&\leq\sum_{i\in A^*}\counter{i}{t-1}\frac{\epsilon_i^2}{2\pa{\sigma_i^2+{\abs{A^*}\varepsilon_i}}}-\sum_{i\in A^*}\counter{i}{t-1}\sigma_i^2\frac{\varepsilon_i^2}{4\pa{\sigma_i^2+{\abs{A^*}\varepsilon_i}}^2}&\frac{\sigma_i^2}{\sigma_i^2+\abs{A^*}\varepsilon_i}\leq 1,\\&=\sum_{i\in A^*}\counter{i}{t-1}\lambda_i\varepsilon_i-\sum_{i\in A^*}\counter{i}{t-1}\sigma_i^2\lambda_i^2\\ &\leq
\sum_{i\in A^*}\counter{i}{t-1}\lambda_i\pa{\mu_i^*-\mean{i,t-1}}-\sum_{i\in A^*}\counter{i}{t-1}\sigma_i^2\lambda_i^2&\text{using}~\eqref{ass:x_i}
,
\\ &=
\sum_{u\in [t-1]}\pa{\pa{\blambda\odot \be_{A_u\cap A^*}}\transpose\pa{\bmu^*-\bX_{u}}- \pa{\blambda\odot \be_{A_u\cap A^*}}\transpose\bD \pa{\blambda\odot \be_{A_u\cap A^*}}},
\end{align*}
where $\bD$ is the diagonal matrix with $D_{ii}=\sigma_i^2$ for all $i\in [\arms]$. 
For all $u\in [t-1],$ since $\blambda\geq 0$, we can write the following axis-realignment inequality
\begin{align*}\pa{\blambda\odot \be_{A_u\cap A^*}}\transpose\bSigma^* \pa{\blambda\odot\be_{A_u\cap A^*}} &=  \sum_{i\in A_u\cap A^*}\sum_{j\in A_u\cap A^*}\Sigma^*_{ij}\lambda_i\lambda_j
\\
&\leq \sum_{i\in A_u\cap A^*}\sum_{j\in A_u\cap A^*}\frac{0\vee\Sigma^*_{ij}}{2}\pa{\lambda_i^2+\lambda_j^2} \\&= \sum_{i\in A_u\cap A^*}\pa{\sum_{j\in A_u\cap A^*}{0\vee\Sigma^*_{ij}}}\lambda_i^2\\&\leq \pa{\blambda\odot \be_{A_u\cap A^*}}\transpose\bD \pa{\blambda\odot \be_{A_u\cap A^*}}.\end{align*}
Thus, we have
\begin{align*}e^{-1}\sum_{i\in A^*}\zeta_i&\leq \sum_{u\in [t-1]}\pa{\pa{\blambda\odot \be_{A_u\cap A^*}}\transpose\pa{\bmu^*-\bX_{u}}- \pa{\blambda\odot \be_{A_u\cap A^*}}\transpose\bSigma^* \pa{\blambda\odot \be_{A_u\cap A^*}}}\\&\leq 
\sum_{u\in [t-1]}\pa{\pa{\blambda\odot \be_{A_u\cap A^*}}\transpose\pa{\bmu^*-\bX_{u}}- \log\EE{e^{\pa{\blambda\odot \be_{A_u\cap A^*}}\transpose\pa{\bX-\bmu^*}}}},
\end{align*}
where the last inequality uses Assumption~\ref{ass:subexp} and $\norm{\blambda\odot \be_{A_u\cap A^*}}_1\leq 1/2$.
Now, notice that $$\EE{\exp\pa{\sum_{u\in [t-1]}\pa{\pa{\blambda\odot \be_{A_u\cap A^*}}\transpose\pa{\bmu^*-\bX_{u}}- \log\EE{e^{\pa{\blambda\odot \be_{A_u\cap A^*}}\transpose\pa{\bX-\bmu^*}}}}}}$$ equals
 \begin{align*}\EE{\prod_{u\in[t-1]}\frac{e^{{\pa{\blambda\odot \be_{A_u\cap A^*}}\transpose\pa{\bmu^*-\bX_{u}}}}}{\EE{e^{\pa{\blambda\odot \be_{A_u\cap A^*}}\transpose\pa{\bX-\bmu^*}}}}}&=\prod_{u\in[t-1]}\EE{\frac{e^{{\pa{\blambda\odot \be_{A_u\cap A^*}}\transpose\pa{\bmu^*-\bX_{u}}}}}{\EE{e^{\pa{\blambda\odot \be_{A_u\cap A^*}}\transpose\pa{\bX-\bmu^*}}}}}\\&=1,\end{align*}
so from  Markov inequality, we get the following bound:
\begin{align*}
\PP{\sum_{u\in [t-1]}\pa{\pa{\blambda\odot \be_{A_u\cap A^*}}\transpose\pa{\bmu^*-\bX_{u}}- \log\EE{e^{\pa{\blambda\odot \be_{A_u\cap A^*}}\transpose\pa{\bX-\bmu^*}}}}\geq e^{-1}\sum_{i\in A^*}\zeta_i}\leq e^{-\sum_{i\in A^*}\zeta_ie^{-1}},\end{align*}
thus, we showed that
\[\PP{\mathfrak{B}_{\bd,t}\cap\bigcap_{i\in A^* }\sset{\counter{i}{t-1}\frac{\pa{\mu_i^*-\mean{i,t-1}}^{+2}}{4\pa{\sigma_i^2+\abs{A^*}\pa{\mu_i^*-\mean{i,t-1}}}}> \zeta_i }}\leq e^{-\sum_{i\in A^*}\zeta_ie^{-1}},\]
i.e.
\[\PP{\bigcap_{i\in A^* }\sset{\II{\mathfrak{B}_{\bd,t}}\counter{i}{t-1}\frac{\pa{\mu_i^*-\mean{i,t-1}}^{+2}}{4\pa{\sigma_i^2+\abs{A^*}\pa{\mu_i^*-\mean{i,t-1}}}}> \zeta_i }}\leq e^{-\sum_{i\in A^*}\zeta_ie^{-1}},\]
By Lemma 8 of \citet{Magureanu2014}, since $\delta\geq em$, we have  
\begin{align*}\PP{\mathfrak{B}_{\bd,t}\cap\mathfrak{A}_{t}}&=\PP{\mathfrak{B}_{\bd,t}\cap\sset{\sum_{i\in A^*}\counter{i}{t-1}\frac{\pa{\mu_i^*-\mean{i,t-1}}^{+2}}{4\pa{\sigma_i^2+\abs{A^*}\pa{\mu_i^*-\mean{i,t-1}}}}\geq \delta }}
\\
&=
\PP{{\sum_{i\in A^*}\II{\mathfrak{B}_{\bd,t}}\counter{i}{t-1}\frac{\pa{\mu_i^*-\mean{i,t-1}}^{+2}}{4\pa{\sigma_i^2+\abs{A^*}\pa{\mu_i^*-\mean{i,t-1}}}}\geq \delta }}
\\
&\leq \pa{\frac{(\delta-1) e}{m}}^me^{1-\delta}.\end{align*}
\end{proof}
\paragraph{Step 3: If $\neg\mathfrak{D}_{t}$ hold}
The regret under this event can be bounded by $8nm^2\Delta_{\max}/\Delta^2$ using exactly the same method as Lemma~2 of \citet{Degenne2016}. 
\paragraph{Step 4: If $\neg\mathfrak{S}_{t}$ hold} From Proposition~\ref{prop:sigmaupper}, the regret under this event is bounded by a universal constant.
\paragraph{Putting it all together}
Finally, we have shown that there exists two universal constant $c,c'$ satisfying the following (we display the scaled back (by $\kappa$) version of the regret bound to get the dependence into $\kappa$)
\begin{align}\nonumber R_T\leq\Delta_{\max}\pa{\frac{n(n-1)}{2}+\frac{8nm^2}{\Delta^2} + c} + c'\log(m+1)\delta(T)\Biggr[&\log(m+1) \sum_{i\in [n]} \frac{\max_{A\in \actions,i\in A}\sum_{j\in A}0\vee\Sigma^*_{ij}}{\Delta_{i,\min}}\\\nonumber
&+\sum_{i,j\in [n]}\kappa\pa{1+\log\pa{\frac{\Delta_{ij,\max}}{\Delta_{ij,\min}}}}\\\nonumber
&+\sum_{i\in [n]}m\kappa\pa{1+\log\pa{\frac{\Delta_{i,\max}}{\Delta_{i,\min}}}}\\
&+\sum_{i,j\in [n]}\frac{\kappa^{4/3}}{\Delta_{ij,\min}^{1/3}}
\Biggl].\label{rel:fullstate_cov}\end{align}
\end{proof}

\section{The bound of Corollary~\ref{cor:upper_cov}}
\label{app:upper_cov2}
The corollary is obtained in the same way as Theorem~\ref{thm:upper_cov}. We can underline the difference that we don't have to construct $n^2$ covariance estimates, but only $n$ (only the $\nu_i^*$'s). On the other hand, as the estimation uses sub-Gaussian variables, we don't use sub-exponential concentration, which removes one term from the previous result.  The obtained bound is
\begin{align}\nonumber R_T\leq\Delta_{\max}\pa{n +\frac{8nm^2}{\Delta^2} + c} + c'\log(m+1)\delta(T)\Biggr[&\log(m+1)  \sum_{i\in [n]}\frac{\nu^*_i(s\wedge m)}{\Delta_{i,\min}}\\ \nonumber
&+\sum_{i\in [n]}m\pa{1+\log\pa{\frac{\Delta_{i,\max}}{\Delta_{i,\min}}}}\\
&+\sum_{i\in [n]}\frac{(s\wedge m)^{2/3}}{\Delta_{i,\min}^{1/3}}
\Biggl],\label{rel:fullstate_sparse}\end{align}
where $c$ and $c'$ are two constants. Notice that to make the first term dominates the others, we must have
$$n(s\wedge m)^{2/3}/\Delta^{1/3}\vee \pa{nm\pa{1+\log\pa{\Delta_{\max}/\Delta}}} \leq c\log(m+1)  \sum_{i\in [n]}\frac{\nu^*_i(s\wedge m)}{\Delta}\CommaBin$$
for some constant $c$, which gives our condition in Corollary~\ref{cor:upper_cov}.
\section{Proof of Theorem~\ref{thm:lower_sparse}}
\label{app:lower_sparse}
\begin{proof}
Consider $\actions$ containing $\arms/m$ disjoint actions $A_1,\dots,A_{\arms/m}$ composed of $m$ arms. 
$\bX$ is constructed as follows:
$\pa{1\vee{s}/{m}}$ different actions are randomly chosen among $\actions$, with equal probability, except the one for action $A_1$, that have an offset of $\delta$. From $$\pa{1\vee{s}/{m}}=\EE{\sum_{A\in \actions}\II{A~\text{is chosen}}}=\pa{\arms/m-1}\pa{\PP{A_1~\text{is chosen}}-\delta}+\PP{A_1~\text{is chosen}},$$ we have $\PP{A_1~\text{is chosen}}=\pa{1\vee{s}/{m}}m/\arms + \delta\pa{1-m/n}$. We pose $X_i=1$ for $i$ spanning the $(s\wedge m)$ first arm of each chosen action (the other components are set to $0$). Remark that $\bX$ is $s-$sparse with this construction. 

 This problem reduces to a standard bandit problem with $\arms/m$ Bernoulli arms. However, we have an additional piece of information, namely that the sum of the means is $s$. Thus, we can't apply the lower bound from \citet{lai1985asymptotically}, since the distribution family has not a product form (changing the mean of one arm, we have to make sure that the sum of the means doesn't change, so we have to change at least another mean).
 Instead, we use the lower bound result from \citet{graves1997asymptotically}, where we can
 increase the mean of one arm $i$ while decreasing the mean of the others.
 Scaling the regret by $(s\wedge m)^{-1}$, we want to upper bound
 \begin{align*}\KL{\Prb_{\frac{1}{(s\wedge m)}\sum_{i\in A_k} X_i}}{\Prb_{\frac{1}{(s\wedge m)}\sum_{i\in A_1} X_i}} &= \kl{\pa{\frac{m}{\arms}\vee\frac{s}{\arms}} - \delta\frac{m}{n}}{\pa{\frac{m}{\arms}\vee\frac{s}{\arms}} + \delta\pa{1-\frac{m}{n} }},\end{align*}
 which corresponds to an arm $i$ that becomes a best arm for the new distribution. We also want to upper bound
 $$ \kl{\pa{\frac{m}{\arms}\vee\frac{s}{\arms}} - \delta\frac{m}{n} - \frac{\delta}{\frac{n}{m}-2}}{\pa{\frac{m}{\arms}\vee\frac{s}{\arms}} - \delta\frac{m}{n}},$$
 which corresponds to the decrease of the mean of each sub-optimal arm $k$ different from $i$ (so that the sum of the mean remain constant).
 We are going to use the inequality $\kl{x}{y}\leq\frac{(x-y)^2}{y(1-y)}$ for all $x,y\in (0,1)$. Since $\frac{ms}{2(n-m)}\geq \Delta=(s\wedge m)\delta$, we have $\delta\frac{m}{n}\leq\delta(1-\frac{m}{n})\leq \pa{\frac{m}{\arms}\vee\frac{s}{\arms}}/2\leq 1/4$,  and thus
 \[\pa{\pa{\frac{m}{\arms}\vee\frac{s}{\arms}} - \delta\frac{m}{n}}\pa{1-\pa{\frac{m}{\arms}\vee\frac{s}{\arms}} + \delta\frac{m}{n}}\geq \pa{\frac{m}{\arms}\vee\frac{s}{\arms}}/4,\]
 \[\pa{\pa{\frac{m}{\arms}\vee\frac{s}{\arms}} + \delta\pa{1-\frac{m}{n} }}\pa{1-{\pa{\frac{m}{\arms}\vee\frac{s}{\arms}} - \delta\pa{1-\frac{m}{n} }}}\geq \pa{\frac{m}{\arms}\vee\frac{s}{\arms}}/4.\]
 Thus, we get the upper bounds
 \begin{align}\label{rel:klbound1}\kl{\pa{\frac{m}{\arms}\vee\frac{s}{\arms}} - \delta\frac{m}{n}}{\pa{\frac{m}{\arms}\vee\frac{s}{\arms}} + \delta\pa{1-\frac{m}{n} }}\leq \frac{4\delta^2}{\pa{\frac{m}{\arms}\vee\frac{s}{\arms}}}\end{align}\begin{align}\label{rel:klbound2}
 \kl{\pa{\frac{m}{\arms}\vee\frac{s}{\arms}} - \delta\frac{m}{n} - \frac{\delta}{\frac{n}{m}-2}}{\pa{\frac{m}{\arms}\vee\frac{s}{\arms}} - \delta\frac{m}{n}}\leq \frac{4\delta^2}{\pa{\frac{n}{m}-2}^2\pa{\frac{m}{\arms}\vee\frac{s}{\arms}}}\cdot\end{align}
 From \citet{graves1997asymptotically}, we have the lower bound
 \[\liminf_{T\to \infty} \frac{R_T}{\log\pa{T}}\geq (s\wedge m)\inf_{\bc}\sum_{k=2}^{\arms/m}\delta c_k,\]
 where the above infimum is over all $c_2,\dots,c_{n/m}$ in $\R_+$ such that for all $i\in \sset{2,\dots,n/m}$,
  \[c_i\kl{\!\pa{\frac{m}{\arms}\!\vee\!\frac{s}{\arms}}\! -\! \delta\frac{m}{n}}{\!\pa{\frac{m}{\arms}\!\vee\!\frac{s}{\arms}} \!+\! \delta\pa{1\!-\!\frac{m}{n} }\!} + \sum_{k=2, k\neq i}^{n/m}\! c_k \kl{\!\pa{\frac{m}{\arms}\!\vee\!\frac{s}{\arms}} \!-\! \delta\frac{m}{n} \!-\! \frac{\delta}{\frac{n}{m}\!-\!2}}{\pa{\frac{m}{\arms}\!\vee\!\frac{s}{\arms}} \!-\! \delta\frac{m}{n}\!}\geq 1.\]
 Using the bounds \eqref{rel:klbound1} and \eqref{rel:klbound2}, we can relax the above constraint as
  \[\forall i\in \sset{2,\dots,n/m},~c_i\frac{4\delta^2}{\pa{\frac{m}{\arms}\vee\frac{s}{\arms}}} + \sum_{k=2, k\neq i}^{n/m} c_k \frac{4\delta^2}{\pa{\frac{n}{m}-2}^2\pa{\frac{m}{\arms}\vee\frac{s}{\arms}}}\geq 1.\]
  By symmetry of the constraint with respect to $c_i$, and by linearity of the objective, there is a maximizer $\bc$ that satisfies $c_1=\dots=c_{n/m}=c$, with 
  $$4c\delta^2\pa{\frac{1}{\pa{\frac{m}{\arms}\vee\frac{s}{\arms}} } + \frac{1}{\pa{\frac{n}{m}-2}\pa{\frac{m}{\arms}\vee\frac{s}{\arms}}}}=1.$$
 Thus, since $\Delta=(s\wedge m)\delta$, we get $$\liminf_{T\to \infty} \frac{R_T}{\log\pa{T}}\geq\frac{s(s\wedge m)\pa{1-2m/\arms}}{4\Delta}\cdot$$
 Notice that we recover the full information case (with a lower bound that equals 0) when $\arms/m=2$, as expected.
\end{proof}

\section{Proof of Lemma~\ref{lem:sparse}}
\label{app:sparse}
\begin{proof}
We use the fact that $\sum_{j\in A}\abs{X_j}\leq (s\wedge m)$.
This gives \begin{align*}\sum_{j\in A}0\vee\Sigma^*_{ij}&= \sum_{j\in A}0\vee\EE{X_iX_j-\mu_i^*\mu_j^*} \\&\leq \sum_{j\in A}\pa{\EE{\abs{X_iX_j}}+\abs{\mu_i^*\mu_j^*}}\\&=\EE{\abs{X_i}\sum_{j\in A}\abs{X_j}}+\abs{\mu_i^*}\sum_{j\in A}\abs{\mu_j^*}
\\&\leq 2\EE{\abs{X_i}}(s\wedge m). \end{align*}
\end{proof}

\section{Confidence regions comparison}
\label{app:cucb}

We give here the two algorithms \textsc{cucb-v} and \textsc{cucb-kl}, which, as we have seen, also matches the lower bound given to the Theorem~\ref{thm:lower_sparse}, in the specific regime where $s\geq m$. Both the two algorithms rely on the same optimization $A_t=\argmax_{A\in \cA} \be_A\transpose \bmu_t$, where the vector $\bmu_t$ is defined for \textsc{cucb-v} as $$\forall i\in [n],\quad\mu_{i,t}\triangleq{1\wedge\pa{\bar\mu_{i,t-1}+\sqrt{\frac{2\zeta\bar\sigma^2_{i,t-1}\log(t)}{ N_{i,t-1}}} + \frac{3\zeta\log(t)}{N_{i,t-1}}}},$$
where $$\bar\sigma^2_{i,t-1} \triangleq \frac{\sum_{t'\in[t-1]}\II{i=i_{t'}}\pa{X_{i,t'}-\mean{i,t-1}}^2}{N_{i,t-1}}\CommaBin$$
and for \textsc{cucb-kl} as
$$\forall i\in [n],\quad\mu_{i,t}\text{ is the unique solution $x$ to }N_{i,t-1}\kl{\bar\mu_{i,t-1}}{x}=\zeta\log(t)  \text{ such that }x\in [\bar\mu_{i,t-1},1].$$
We take $\zeta=1.2$ (although all $\zeta>1$ are valid). 
The algorithms above can also be seen as a bilinear maximization where $\bmu_t$ is maximized over a confidence region that is a Cartesian product one $1$-dimensional confidence intervals. 
 We illustrate in Figure~\ref{fig:conf_ellipsoid} the difference between the confidence region considered in \textsc{escb-c} (when the correlation is low) and \textsc{cucb-kl}. The red points represent $\bmu_t$ for each region. 
It can be seen that the Cartesian product confidence region greatly overestimates the risk in directions that are not close to the axes, giving rise to over-exploration. It is important to note however that this price to pay can be interesting in practice, because the corresponding algorithms are then very efficient (LP over $\cA$, supposed possible\footnote{Otherwise an approximation regret would be a more appropriate performance measure to consider.}). As we noted in Remark~\ref{rk:intersect}, considering the intersection between the two confidence regions gives rise to an even tighter region, and therefore a better regret bound. 

\begin{figure}[H]
\centering
\begin{tikzpicture}[scale=1]
\draw [-] (1.23,0.45) -- (.84,.84);
\draw [-] (1.35,0.92) -- (1.135,1.135);
\draw[->] (0,0) -- (1.7,0);
\draw (1.7,0) node[ right] {$\be_1$};
\draw [->] (0,0) -- (0,1.7);
\draw (0,1.7) node[right] {$\be_2$};
\draw [->] (0,0) -- (1.7,1.7);
\draw (1.7,1.7) node[right] {$\be_{\sset{1,2}}$};
\draw (-1.35,-.92) rectangle (1.35,.92);
\draw (0,0) node[ below] {$\vmean{t-1}$};
\draw[scale=0.8,domain=-.815:.815,smooth,variable=\x] plot ({1.5*\x},{0.5*(0.4*(1-\x*\x/(1+abs(\x)*0.4)) + sqrt( (0.4*(1-\x*\x/(1+abs(\x)*0.4)))^2+4*(1-\x*\x/(1+abs(\x)*0.4)) ))});
\draw[scale=0.8,domain=-.815:.815,smooth,variable=\x] plot ({1.5*\x},{-0.5*(0.4*(1-\x*\x/(1+abs(\x)*0.4)) + sqrt( (0.4*(1-\x*\x/(1+abs(\x)*0.4)))^2+4*(1-\x*\x/(1+abs(\x)*0.4)) ))});

\draw[scale=0.8,domain=-.815:.815,smooth,variable=\x] plot ({0.75*(0.4*(1-\x*\x/(1+abs(\x)*0.4)) + sqrt( (0.4*(1-\x*\x/(1+abs(\x)*0.4)))^2+4*(1-\x*\x/(1+abs(\x)*0.4)) ))},{\x});
\draw[scale=0.8,domain=-.815:.815,smooth,variable=\x] plot ({-0.75*(0.4*(1-\x*\x/(1+abs(\x)*0.4)) + sqrt( (0.4*(1-\x*\x/(1+abs(\x)*0.4)))^2+4*(1-\x*\x/(1+abs(\x)*0.4)) ))},{\x});
\node [black] at (0,0) {\scalebox{.5}{\textbullet}};
\node [red] at (1.23,0.45) {\scalebox{.5}{\textbullet}};
\node [red] at (1.35,0.92) {\scalebox{.5}{\textbullet}};
\end{tikzpicture}
\caption{\textit{Confidence regions build by  \textsc{escb-c} (the pseudo-ellipse), and \textsc{cucb-kl} (the rectangle), for $\norm{\cdot}_1$ constrained outcomes. Notice that \textsc{cucb-kl} has slightly better confidence intervals along the axis, but that \textsc{escb-c} is better in the direction $\be_{\sset{1,2}}$.}}
\label{fig:conf_ellipsoid}
\end{figure}
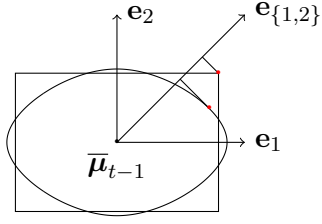

 \section{Implementation using the Lov\'asz extension}
 \label{app:super}
 From the step 1 of the proof of Theorem~\ref{thm:upper_cov}, we have that
 $$\be_{A_t}\transpose{\bmu_t}\leq\be_{A_t}\transpose\bar\bmu_{t-1}+ 2{ \sqrt{\delta(t)\sum_{i\in A_t}\frac{ \sum_{j\in A_t}0\vee\Sigma_{ij,t}} {\counter{i}{t-1}}}}+4m\delta(T)\sqrt{\sum_{i\in A_t}\frac{1}{N_{i,t-1}^2}}\cdot$$
Since the final bound of Theorem~\ref{thm:upper_cov} relies on the above upper bound, in Algorithm~\ref{algo:cov}, instead of maximizing $A\mapsto \max_{\bmu\in \cC_t(A)}\be_{A}\transpose{\bmu}$, we can maximize
$$A\mapsto \be_{A}\transpose\bar\bmu_{t-1}+ 2{ \sqrt{\delta(t)\sum_{i\in A}\frac{ \sum_{j\in A}0\vee\Sigma_{ij,t}} {\counter{i}{t-1}}}}+4m\delta(T)\sqrt{\sum_{i\in A}\frac{1}{N_{i,t-1}^2}}\cdot$$
Our goal here is to provide a continuous extension of the above set function that is concave on $[0,1]^n$, and thus efficient to maximize.
The linear term trivially extends to the linear function $\bx\mapsto \bx\transpose\bar\bmu_{t-1} $. The last two term can be extended relying on the Lov\'asz extension \citep{lovasz1983submodular}.
 We recall that the Lov\'asz extension of a set function $f$  is defined as $f^L(\bx)\triangleq\EE{f\pa{\sset{i\in [n],~x_i\geq U}}}$, where the expectation is over $U\sim\cU[0,1]$.
 The Lov\'asz extension is concave if and only if $f$ is a supermodular function \citep{lovasz1983submodular}, i.e., $$f(A)+f(B)\leq f(A\cup B)+f(A\cap B)\quad\forall A,B\subset [n].$$ 
 It is easy to check that a function $G:A\mapsto \sum_{i\in A}\sum_{j\in A} a_{ij}$ is supermodular for $a_{ij}\geq 0$, so its Lov\'asz extension is concave. Composing by the square root, we thus have a concave extension of the second and last term.

After the maximization of the extension, a continuous maximizer $\bx_t$ is returned, and the agent plays $A_t=\sset{i\in [n],~x_{i,t}\geq U}$ where $U\sim\cU[0,1]$.
The analysis holds the same, except in Proposition~\ref{prop:sumcounter}, where counters are updated only for the chosen set. 
Let $\sigma_t$ be a permutation such that $x_{\sigma_t(1),t}\geq \dots,\geq x_{\sigma_t(n),t}$. Then, the set $S_j=\sset{\sigma_t(1),\dots,\sigma_t(j)}$ is chosen with probability $p_{j,t}=x_{\sigma_t(j)}-x_{\sigma_t(j+1)}$ (with the convention $x_{\sigma_t(n+1)}=0$). The continuous extension evaluated at $\bx_t$ is of the form
$$\sum_j p_{j,t} \be_{S_j}\transpose \bar\bmu_{t-1} + \sqrt{\sum_j p_{j,t}G_1(S_j)}+ \sqrt{\sum_j p_{j,t}G_2(S_j)},$$
where $G_1$ and $G_2$ are the supermodular functions corresponding to the second and last term respectively. Since the probabilities $p_{j,t}$ are inside the square root, applying the \emph{Probabilistically triggered arms} setting of \citet{wang2017improving}
gives an extra factor of $1+\log\pa{\frac{m\log(T)}{\Delta^2}}$. 

\section{General stochastic combinatorial semi-bandits results}\label{app:extra_cmab}

\begin{theorem}[Regret bound for $\ell_2$-norm error]\label{thm:dege}Let $I$ be a set of index. For all $i\in I$, let $\pa{\alpha_i,\beta_{i,T}}\in [1/2,1)\times \R_+$. Let $I_t$ be a subset of $I$ such that for all $i\in I_t$, $N_{i,t}=N_{i,t-1}+1$. We pose $\Delta_t\triangleq\Delta(A_t)$.
For $t\geq 1$,
consider the event
$$\fA_t\triangleq\sset{\Delta_t\leq
\norm{\sum_{i\in I_t}\frac{\beta_{i,T}^{\alpha_i}\be_i}{N_{i,t-1}^{\alpha_i}}}_2}.$$
Then, $$\sum_{t=1}^T\II{\fA_t}\Delta_t\leq
4\log_2(4\sqrt{m})\sum_{i\in I}{\beta_{i,T}\eta_i},$$

where

\begin{align*}\eta_i\triangleq
 \left\{
    \begin{array}{ll}
        {8\log_2(4\sqrt{m})}{\Delta_{i,\min}^{-1}} &\mbox{if } \alpha_i=1/2 \\
        \pa{\pa{2^{-\frac{1}{\alpha_i}}-2^{-2}}\pa{1-\alpha_i}\Delta_{i,\min}^{\frac{1-\alpha_i}{\alpha_i}}}^{-1} & \mbox{if }  1/2<\alpha_i<1 \\
         4\pa{1+\log\pa{\frac{\Delta_{i,\max}}{\Delta_{i,\min}}}}& \mbox{if }  \alpha_i=1.
    \end{array}
\right.
\end{align*}
\end{theorem}

\begin{proof} Let $t\geq 1$. We define $\Lambda_t\triangleq \norm{\sum_{i\in I_t}{\beta_{i,T}^{\alpha_i}}{N_{i,t-1}^{-\alpha_i}\be_i}}_2$. 
We start by a simple lower bound on $\Lambda_t$, holding for any $j\in I_t$, \begin{align}\Lambda_t\geq \norm{\frac{\beta_{j,T}^{\alpha_j}\be_j}{N_{j,t}^{\alpha_j}}}_2=\frac{\beta_{j,T}^{\alpha_j}}{N_{j,t}^{\alpha_j}}.\label{LowerLambda_t}\end{align}
We then use the same reverse amortisation technique than in \citet{wang2017tighter}.  
\begin{align*}
    \Lambda_t&=-\Lambda_t+\norm{\sum_{i\in I_t}\frac{2\beta_{i,T}^{\alpha_i}\be_i}{N_{i,t-1}^{\alpha_i}}}_2
    \\
    &=
    -\norm{\frac{\Lambda_t\be_{I_t}}{\norm{\be_{I_t}}_2}}_2+\norm{\sum_{i\in I_t}\frac{2\beta_{i,T}^{\alpha_i}\be_i}{N_{i,t-1}^{\alpha_i}}}_2
        \\
    &\leq 
    \norm{\sum_{i\in I_t}\pa{\frac{2\beta_{i,T}^{\alpha_i}}{N_{i,t-1}^{\alpha_i}}-\frac{\Lambda_t}{\norm{\be_{I_t}}_2}}^+\be_i }_2 
          \\
              &= 
    \norm{\sum_{i\in I_t}\pa{\frac{2\beta_{i,T}^{\alpha_i}}{N_{i,t-1}^{\alpha_i}}-\frac{\Lambda_t}{\norm{\be_{I_t}}_2}}^+\II{\Lambda_t\geq \frac{\beta_{i,T}^{\alpha_i}}{N_{i,t-1}^{\alpha_i}} }\be_i }_2 &\text{Using } \eqref{LowerLambda_t}
          \\
    &\leq \norm{\sum_{i\in I_t}\II{2\Lambda_t\geq\frac{2\beta_{i,T}^{\alpha_i}}{N_{i,t-1}^{\alpha_i}}\geq \frac{\Lambda_t}{\norm{\be_{I_t}}_2}}\frac{2\beta_{i,T}^{\alpha_i}\be_i}{N_{i,t-1}^{\alpha_i}} }_2.
\end{align*}
We now decompose the interval $[2,{1}/{\norm{\be_{I_t}}_2}]$ using a peeling:
$$[2,{1}/{\norm{\be_{I_t}}_2}]\subset\bigcup_{k=0}^{\ceil{\log_2\pa{\norm{\be_{I_t}}_2}}} [2^{1-k},2^{-k}].$$
This induces a partition of the set of indices:
$$\II{i\in I_t,~2\Lambda_t\geq\frac{2\beta_{i,T}^{\alpha_i}}{N_{i,t-1}^{\alpha_i}}\geq \frac{\Lambda_t}{\norm{\be_{I_t}}_2}}\subset \bigcup_{k=0}^{\ceil{\log_2\pa{\norm{\be_{I_t}}_2}}} J_{k,t}, $$
where for all interger $1\leq k\leq {\ceil{\log_2\pa{\norm{\be_{I_t}}_2}}} $,
$$J_{k,t}\triangleq \sset{i\in I_t,~2^{1-k}\Lambda_t\geq\frac{2\beta_{i,T}^{\alpha_i}}{N_{i,t-1}^{\alpha_i}}\geq {2^{-k}\Lambda_t}}.$$
We can thus
upper bound $\Lambda_t^2$ using this decomposition
\begin{align*}\Lambda_t^2&\leq \norm{\sum_{i\in I_t}\II{2\Lambda_t\geq\frac{2\beta_{i,T}^{\alpha_i}}{N_{i,t-1}^{\alpha_i}}\geq \frac{\Lambda_t}{\norm{\be_{I_t}}_2}}\frac{2\beta_{i,T}^{\alpha_i}\be_i}{N_{i,t-1}^{\alpha_i}} }_2^2
\\&\leq \sum_{k=0}^{\ceil{\log_2\pa{\norm{\be_{I_t}}_2}}}\norm{\sum_{i\in J_{k,t}}{\frac{2\beta_{i,T}^{\alpha_i}\be_i}{N_{i,t-1}^{\alpha_i}}}}^2_2
\\&\leq 
\sum_{k=0}^{\ceil{\log_2\pa{\norm{\be_{I_t}}_2}}}2^{2-2k}\Lambda_t^2\norm{{{\be_{J_{k,t}}}} }_2^2.
\end{align*}
This last inequality implies that there must exist one integer $k_t$ such that $\abs{J_{k_t,t}}=\norm{\be_{J_{k_t,t}}}^2_2\geq 2^{2k_t-2}\pa{1+\ceil{\log_2\pa{\norm{\be_{I_t}}_2}}}^{-1}$. We now upper bound ${\sum_{t=1}^T\II{\fA_t}\Delta_t}$, using $\abs{I_t}\leq m$, i.e., $$\ceil{\log_2\pa{\norm{\be_{I_t}}_2}}\leq \ceil{\log_2(m)/2}.$$
\begin{align*}{\sum_{t=1}^T\II{\fA_t}\Delta_t}&\leq{\sum_{t=1}^T\sum_{k=0}^{\ceil{\log_2(m)/2}}\II{k_t=k,~\fA_t}\Delta_t}
\\&\leq
{\sum_{t=1}^T\sum_{k=0}^{\ceil{\log_2(m)/2}}{\II{k_t=k,~\fA_t}\sum_{i\in I}\II{i\in J_{k,t}}}\Delta_t2^{2-2k}\pa{\ceil{\log_2(m)/2}+1}}
\\&\leq
{\sum_{t=1}^T\sum_{k=0}^{\ceil{\log_2(m)/2}}{\sum_{i\in I}\II{i\in I_t,~N_{i,t-1}^{\alpha_i}\leq \frac{2^{k+1}\beta_{i,T}^{\alpha_i}}{\Delta_t}}}\Delta_t2^{2-2k}\pa{\ceil{\log_2(m)/2}+1}}
\\&=
\pa{\ceil{\log_2(m)/2}+1}\sum_{k=0}^{\ceil{\log_2(m)/2}}2^{2-2k}\sum_{i\in I}\underbrace{{\sum_{t=1}^T{\II{i\in I_t,~N_{i,t-1}^{\alpha_i}\leq \frac{2^{k+1}\beta_{i,T}^{\alpha_i}}{\Delta_{t}}}}\Delta_t}}_{\numterm{sumcounter}_{i,k}}.
\end{align*}

\vspace{-.3cm}

\noindent
Applying Proposition~\ref{prop:sumcounter} gives
$$\eqref{sumcounter}_{i,k}\leq\II{\alpha_i<1} \frac{\beta_{i,T}2^{\frac{k+1}{\alpha_i}}}{1-\alpha_i}\Delta_{i,\min}^{1-1/\alpha_i}  +   \II{\alpha_i=1}2^{k+1}\beta_{i,T}\pa{1+\log\pa{\frac{\Delta_{i,\max}}{\Delta_{i,\min}}}}
.$$
So we get, using $\ceil{\log_2(m)/2}+1\leq \log_2(4\sqrt{m})$,

$$\sum_{t=1}^T\II{\fA_t}\Delta_t\leq 4\log_2(4\sqrt{m})\sum_{i\in I}{\beta_{i,T}\eta_i},$$
\begin{align*}\text{with }\eta_i=
 \left\{
    \begin{array}{ll}
        {8\log_2(4\sqrt{m})}{\Delta_{i,\min}^{-1}} &\mbox{if } \alpha_i=1/2 \\
        \pa{\pa{2^{-\frac{1}{\alpha_i}}-2^{-2}}\pa{1-\alpha_i}\Delta_{i,\min}^{\frac{1-\alpha_i}{\alpha_i}}}^{-1} & \mbox{if }  1/2<\alpha_i<1 \\
         4\pa{1+\log\pa{\frac{\Delta_{i,\max}}{\Delta_{i,\min}}}}& \mbox{if }  \alpha_i=1.
    \end{array}
\right.
\end{align*}
\end{proof}
\begin{prop} Let $i\in I$ and $f_i: \R_+\to \R_+$ be a non increasing function, integrable on $[\Delta_{i,\min},\Delta_{i,\max}]$. Then
$$\sum_{t=1}^T{\II{i\in I_t,~ N_{i,t-1}\leq f_i(\Delta_t)}}{\Delta_t} \leq
{f_i(\Delta_{i,\min})}\Delta_{i,\min}+
\int_{\Delta_{i,\min}}^{\Delta_{i,\max}}{f_i(x)\emph{d}x}.$$
In particular, \\
\begin{itemize}
\item If $f_i(x)=\beta_{i,T}x^{-1/\alpha_i}$, $\alpha_i\in (0,1)$ and $\beta_{i,T}\geq 0$, then 
\begin{align*}\sum_{t=1}^T{\II{i\in I_t,~ N_{i,t-1}\leq f_i(\Delta_t)}}{\Delta_t} &\leq
\Delta_{i,\min}^{1-1/\alpha_i}\frac{\beta_{i,T}}{1-\alpha_i}-\Delta_{i,\max}^{1-1/\alpha_i}\frac{\alpha_i\beta_{i,T}}{1-\alpha_i}\\&\leq \Delta_{i,\min}^{1-1/\alpha_i}\frac{\beta_{i,T}}{1-\alpha_i}
.\end{align*}

\item If $f_i(x)=\beta_{i,T}x^{-1}$, $\beta_{i,T}\geq 0$, then 
$$\sum_{t=1}^T{\II{i\in I_t,~ N_{i,t-1}\leq f_i(\Delta_t)}}{\Delta_t} \leq\beta_{i,T}\pa{1+\log\pa{\frac{\Delta_{i,\max}}{\Delta_{i,\min}}}}
.$$
\end{itemize}
  \label{prop:sumcounter}
 \end{prop}
\begin{proof}
Consider $\Delta_{i,\max}=\Delta_{i,1}\geq \Delta_{i,2}\geq \dots \geq\Delta_{i,K_i}=\Delta_{i,\min}$ being all possible values for $\Delta_t$ when $i\in I_t$.
We define a dummy gap $\Delta_{i,0}=\infty$ and let $f_i\pa{\Delta_{i,0}}=0$.
In \eqref{rangebreak}, we first break the range $(0,f_i(\Delta_t)]$ of the counter $N_{i,t-1}$ into sub intervals: $$(0,f_i(\Delta_t)]=(f_i(\Delta_{i,0}),f_i(\Delta_{i,1})]\cup\dots\cup (f_i(\Delta_{i,k_t-1}),f_i(\Delta_{i,k_t})],$$ where $k_t$ is the index such that $\Delta_{i,k_t} = \Delta_t$. {This index $k_t$ exists by assumption that the subdivision contains all possible values for $\Delta_t$ when $i\in I_t$. Notice that in \eqref{rangebreak}, we do not explicitly use $k_t$, but instead sum over all $k\in [K_i]$ and filter against the event $\sset{\Delta_{i,k}\geq \Delta_t}$, which is equivalent to summing over $k\in [k_t].$}  

\begin{align}
&\sum_{t=1}^T{\II{i\in I_t,~ N_{i,t-1}\leq f_i(\Delta_t)}}{\Delta_t}\nonumber
\\&=\label{rangebreak}
\sum_{t=1}^T\sum_{k=1}^{K_i}\II{i\in I_t,~f_i(\Delta_{i,k-1})< N_{i,t-1}\leq f_i(\Delta_{i,k}),\Delta_{i,k}\geq \Delta_t}{\Delta_t}.
\end{align}
Over each event that $N_{i,t-1}$ belongs to the interval $(f_i(\Delta_{i,k-1}),f_i(\Delta_{i,k})]$, we upper bound the suffered gap $\Delta_t$ by $\Delta_{i,k}$. 

\begin{align}
\eqref{rangebreak}&\leq\label{uppergap}
\sum_{t=1}^T\sum_{k=1}^{K_i}\II{i\in I_t,~f_i(\Delta_{i,k-1})< N_{i,t-1}\leq f_i(\Delta_{i,k}),\Delta_{i,k}\geq \Delta_t}{\Delta_{i,k}}.
\end{align}
Then, we further upper bound the summation by adding events that $N_{i,t-1}$ belongs to the remaining intervals
$(f_i(\Delta_{i,k-1}),f_i(\Delta_{i,k})]$ for $k_t<k\leq K_i$, associating them to a suffered gap $\Delta_{i,k}$. This is equivalent to removing the filtering against the event $\sset{\Delta_{i,k}\geq \Delta_t}$. 

\begin{align}
\eqref{uppergap}&\leq\label{addremainingintervals}
\sum_{t=1}^T\sum_{k=1}^{K_i}\II{i\in I_t,~f_i(\Delta_{i,k-1})< N_{i,t-1}\leq f_i(\Delta_{i,k})}{\Delta_{i,k}}.\end{align}
Now, we invert the summation over $t$ and the one over $k$. 

\begin{align}
\eqref{addremainingintervals}&=\label{sumswitch}
\sum_{k=1}^{K_i}\sum_{t=1}^T\II{i\in I_t,~f_i(\Delta_{i,k-1})< N_{i,t-1}\leq f_i(\Delta_{i,k})}{\Delta_{i,k}}.
\end{align}
For each $k\in[K_i]$, the number of times $t\in [T]$ that the counter $N_{i,t-1}$ belongs to $(f_i(\Delta_{i,k-1}),f_i(\Delta_{i,k})]$ can be upper bounded  by the number of integers in this interval. This is due to the event $\sset{i\in I_t}$, imposing that $N_{i,t-1}$ is incremented, so $N_{i,t-1}$ cannot be worth the same integer for two different times $t$ satisfying $i\in I_t$. We use the fact that for all $x,y\in \R$, $x\leq y$, the number of integers in the interval $(x,y]$ is exactly $\floor{y}-\floor{x}$. 

\begin{align}
\eqref{sumswitch}&\leq\label{intcount}
\sum_{k=1}^{K_i}\pa{\floor{f_i(\Delta_{i,k})}-\floor{f_i(\Delta_{i,k-1})}}{\Delta_{i,k}}.
\end{align}
We then simply expand the summation, and some terms are cancelled (remember that $f_i\pa{\Delta_{i,0}}=0$).

\begin{align}
\eqref{intcount}&=\label{expand}
\floor{f_i(\Delta_{i,K_i})}\Delta_{i,K_i}+
\sum_{k=1}^{K_i-1}\floor{f_i(\Delta_{i,k})}\pa{\Delta_{i,k}-\Delta_{i,k+1}}
\end{align}
We use $\floor{x}\leq x$ for all $x\in \R$. Finally, we recognize a right Riemann sum, and use the fact that $f_i$ is non increasing
to upper bound each $f_i(\Delta_{i,k})\pa{\Delta_{i,k}-\Delta_{i,k+1}}$ by $\int_{\Delta_{i,k+1}}^{\Delta_{i,k}}f_i(x)\text{d}x$, for all $k\in [K_i-1]$.
\begin{align}
\eqref{expand}&\leq\label{upperfloor}
{f_i(\Delta_{i,K_i})}\Delta_{i,K_i}+
\sum_{k=1}^{K_i-1}{f_i(\Delta_{i,k})}\pa{\Delta_{i,k}-\Delta_{i,k+1}}
\\&\leq\label{Riemann}
{f_i(\Delta_{i,K_i})}\Delta_{i,K_i}+
\int_{\Delta_{i,K_i}}^{\Delta_{i,1}}{f_i(x)\text{d}x}.
\end{align}
\end{proof}

\end{document}